\documentclass[runngingheads]{llncs}
\usepackage[T1]{fontenc}
\usepackage[utf8]{inputenc}
\usepackage{amsmath,amssymb}
\usepackage{cleveref}
\usepackage{graphicx} 
\usepackage{caption}
\usepackage{xcolor}
\usepackage[normalem]{ulem}
\usepackage{tikz}
\usetikzlibrary{positioning, arrows.meta, patterns}
\usepackage{wrapfig}
\usepackage [linesnumbered,ruled,vlined]{algorithm2e}

\usepackage{etoolbox}
\AtBeginEnvironment{thebibliography}{\def\url#1{}}

\newcommand\recht\operatorname
\newcommand{\tOR}{\mathtt{OR}}
\newcommand{\tAND}{\mathtt{AND}}
\newcommand{\tBE}{\mathtt{BE}}
\newcommand{\R}{{\it Root}}
\newcommand{\ch}{{\it Inp}}
\newcommand{\outp}{{\it Outp}}
\newcommand{\BB}{\mathbb{B}}
\newcommand{\tzero}{0}
\newcommand{\tone}{1}
\newcommand{\BE}[1]{\recht{BE}_{#1}}
\newcommand{\struc}[1]{{\Phi}_{#1}}

\newcommand{\X}{X}
\newcommand{\U}{\mathsf{U}}

\newcommand{\AC}[1]{{\bf{AC#1}}}
\newcommand{\ACo}{{\bf AC-o}}
\newcommand{\ACu}{{\bf AC-u}}
\newcommand{\ACm}{{\bf AC-m}}
\newcommand{\V}{\mathsf{V}}
\newcommand{\iV}{\V^{{\textrm{i}}}}
\newcommand{\cV}{\V^{{\textrm{c}}}}
\newcommand{\E}{\mathsf{E}}

\newcommand{\lock}{\sf LF}
\newcommand{\electricity}{\sf EF}
\newcommand{\alert}{{\sf AF}}
\newcommand{\public}{{\sf PF}}
\newcommand{\operator}{{\sf OF}}
\newcommand{\ExElectricity}{{\sf B_{EF}}}
\newcommand{\ExPublic}{{\sf B_{PF}}}
\newcommand{\ExOperator}{{\sf B_{OF}}}
\newcommand{\doorbellcm}{{\sf doorbell}}
\newcommand{\fElectricity}{\sf electricity fails}
\newcommand{\fPublic}{\sf public fails}
\newcommand{\fOperator}{\sf operator fails}
\newcommand{\fAlert}{\sf alert fails}
\newcommand{\fLock}{\sf lock fails}

\newcommand{\inodes}{\U^{{\textrm{i}} }}
\newcommand{\cnodes}{\U^{\textrm{c}}}

\newcommand{\hlbox}[1]{%
  \smallskip\begin{center}
  \fboxrule1pt\fboxsep3pt\fcolorbox{black!45}{black!8}{%
  \begin{minipage}{.96\linewidth}#1\end{minipage}}
  \end{center}\smallskip}

\newcommand{\GCmar}[1]{\marginremark{blue}{white}{\tiny{[GC]~ #1}}}

\newcommand{\MSinl}[1]{}
\newcommand{\colorpar}[3]{}
\newcommand{\marginremark}[3]{}
\newcommand{\MSmar}[1]{}
\newcommand{\YD}[1]{}
\newcommand{\MLZ}[1]{}
\newcommand{\gc}[1]{}

\newcommand{\myparagraph}[1]{\medskip\noindent{\bf #1}}

\usepackage{xparse}

\NewDocumentCommand{\mylabelparagraph}{m o}{%
  \medskip
  \noindent
  \textbf{#1}%
  \IfValueT{#2}{\label{#2}}%
}

\begin{document}
%
\title{Actual causality in fault trees}


\author{Georgiana Caltais\inst{1}\orcidID{0000-0002-8653-2299}\\ 
Milan {Lopuha\"a-Zwakenberg}\inst{1}\orcidID{0000-0001-5687-854X}\\ 
Mari\"elle {Stoelinga}\inst{1,2}\orcidID{0000-0001-6793-8165}}

\authorrunning{G. Caltais et al.}

\institute{University of Twente, The Netherlands 
\and
Radboud University, The Netherlands\\
\email{\{g.g.c.caltais,m.a.lopuhaa,m.i.a.stoelinga\}@utwente.nl}}

\maketitle              
\begin{abstract}
Fault trees are a widely used as effective risk models for complex systems, answering the question \emph{what can go wrong?}, especially through minimal cut set analysis. We study fault trees from the perspective of Halpern \& Pearl's theory of actual causality. This allows us to use fault trees to answer the question \emph{why has it gone wrong?}, which is fundamental to failure diagnostics. We give a complete classification of each of the different notions of actual causality in terms of the fault tree's graph structure and logical structure, and show how minimal cut sets give rise to actual causes. Furthermore, we discuss the complexity of computing causality in fault trees, and develop algorithms to do so.
\end{abstract}
\section{Introduction} \label{sec:intro}

\begin{wrapfigure}[17]{r}{0.5\textwidth}
    \centering
\vspace{-6em}
\includegraphics[width=0.4\textwidth]{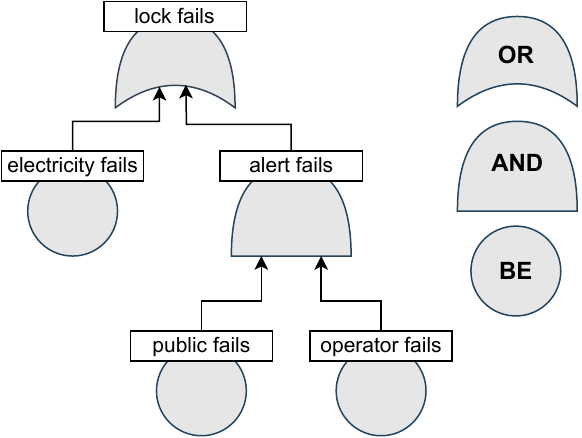}
\vspace{-0.5em}
    \caption{FT modeling a \emph{visdeurbel}, a system supporting migrating fish in the waterways of Utrecht: public observers and operators can both ring a door bell to open the lock when they see a fish on an underwater camera. The lock fails to open (OR-gate) if either electricity fails (basic event, BE) or if the alert fails (AND-gate), which happens if both the public and the operator fail to act.}
    \label{fig:FT-doorbell-expl}
\end{wrapfigure}

The \emph{visdeurbel} (“fish doorbell”) is a citizen-science initiative designed to support migrating fish as they pass through the waterways of Utrecht, a major Dutch city \cite{visdeurbel2023,smithsonian2023visdeurbel}. An underwater camera provides a live stream of fish in front of the lock. The lock operator can be alerted in two ways. Public observers may ring the virtual doorbell when they see fish, or the operator can watch the camera feed directly. The system aids fish migration and raises public awareness of ecology.
This example illustrates how different components and failures (electric, public and operator failures) contribute to a system-level failure (the lock failing to open).


Fault trees (FTs) are a structured, hierarchical model to analyse failure propagation in complex systems~\cite{ruijters2015fault,stoelinga2025concise}.
FTs enable a top-down, deductive approach to failure analysis, breaking down system failures into combinations of component failures using logical gates (e.g., AND, OR). 
In FT-terminology, the inputs of such a gate are often called the \emph{immediate causes} of the higher-level cause labeled by the gate \cite{Vesely1981FaultTreeHandbook}.
Widely used in safety-critical industries such as aerospace \cite{bozzano2017formal}, nuclear energy \cite{guohua2020review}, and cybersecurity \cite{nagaraju2017survey}, FTs enable structured reasoning about system reliability, risk assessment, and failure propagation. They serve as a powerful design tool, helping engineers identify weak points and prioritise mitigation strategies.
An FT for the ``fish doorbell''  is given in Fig.~\ref{fig:FT-doorbell-expl}.

As a risk model, FTs are primarily used to analyze potential failures: \emph{what can go wrong?} This is done via minimal cut set (MCS) analysis: an MCS is a set of basic failure events that need to happen simultaneously in order for the system to fail. After finding a FT's MCSs, these are then analyzed in terms of size, likelihood, and common causes to determine the system's overall reliability.

\noindent \textbf{Actual causality in the Halpern-Pearl framework.}
 Causality, on the other hand, asks the question \emph{why has it gone wrong?} after the fact. In computer science, the framework of actual causality (AC) by Halpern and Pearl \cite{HalpernP01} has been especially popular. In AC, a system is modeled as a \emph{causal model}: a directed acyclic graph of variables, with functions describing each variable in terms of its input. The system's context is given by values of the initial variables. Within such a system one can perform \emph{interventions} using \emph{do-calculus}~\cite{DBLP:conf/uai/Pearl12} to forcibly set certain variables. This allows one to go beyond correlations and define causal relationships based on \emph{counterfactuals}, 
 describing what would have 
 happened if some conditions were different.
This has been applied successfully to problems such as fault localization \cite{baah2010causal}, debugging \cite{degano1999causality}, and failure reasoning \cite{kushnir2016causality}.

\noindent \textbf{Incident analysis and FTs.} 
While FTs are mostly used to prevent failures through risk assessment, they are also an important tool in \emph{incident analysis} or \emph{diagnostics}: when a system fails, understand what has happened and why, to prevent its recurrence. A formal framework of causality would be an important asset towards systematic incident analysis. While the use of FTs in diagnoses has been well-studied with systematic methods \cite{Hurdle2007SystemFaultDiagnostics,LiuLiLidiagnostic08,Papadopoulos2003ModelbasedSM,QianZhengCaoDiesel05},
\emph{\textbf{a systematic application of actual causality to fault trees is missing from the current literature.}} 
So far, the connection between AC and FTs has been limited  to generating FTs from trace data \cite{DBLP:journals/ijccbs/Leitner-FischerL13}, and translating \emph{attack trees} ($\approx$FTs for cybersecurity) to causal models, without investigating their actual causes \cite{DBLP:conf/gramsec/IbrahimRSAP20}.\MLZ{Georgiana, please check}\GCmar{Checked, looks good!}

\noindent \textbf{Contributions.} Our main contribution is to apply the AC framework to FTs.
We consider the most common type of FTs, namely static, coherent FTs.
These can be phrased as directed acyclic graphs of Boolean variables, that are functions of their inputs via their role as AND/OR-gates. Thus FTs can directly be studied from an AC perspective. This by itself is not surprising, but the fact that all functions are Boolean and non-decreasing leads to a deeper understanding than would be possible in regular causal models.

First, it gives us concise classifications of AC. For the three main causality definitions in the AC literature~\cite{ActualCausalityHalpern} (\emph{original}: \ACo{}, \emph{updated}: \ACu{}, \emph{modified}: \ACm{}), we classify the actual causes of event failure in FTs, given a context, allowing for a straightforward causal analysis of FTs (Thms. \ref{thm:maino}, \ref{thm:mainu} \& \ref{thm:mainm}). 
These classifications do not only take into account the Boolean nature of FTs, but also their graph structure. By contrast, FT analysis typically only considers the underlying Boolean functions. This leads to the surprising fact that \emph{\textbf{equivalent FTs can behave differently as causal models.}} From an AC perspective, this is a feature rather than a bug: in AC, intermediate gates are not just dummy variables, but represent real-world events or subsystems that can be targets for interventions. For FTs, \ACo{}, \ACu{}, \ACm{} form a spectrum for defining causality, from a larger emphasis on the graph structure to a complete emphasis on the Boolean nature. As {\ACm} looks only at the Boolean nature of the FT, the relation between {\ACm} and MCSs is stronger. 
In fact, our results allow us to classify MCSs in terms of \ACm{} (Cor.~\ref{cor:mcs}).

Second, we look at the relation between minimal cut sets and \ACo{} and \ACu{}. As MCSs are \emph{potential} pathways to failure, and AC are events that contributed to failure in a specific scenario, we cannot expect them to be the same. In fact, MCSs are typically large in robust systems to prevent a single-point-of-failure, but ACs are typically small: in fact, it is known for {\ACo}, and we show for {\ACu}, that actual causes are only singletons. Nevertheless, there is a clear relation between MCSs and AC. We show that if we take an MCSs as context, then all of its elements are ACs (Thm. \ref{thm:mcs}). Furthermore, for {\ACo} and {\ACu}, in a given context, we show that every element of every MCS that has happened, is an AC (Thm. \ref{thm:mcs2}). If the FT is a (graph-theoretic) tree rather than a DAG,
or if it is in disjunctive normal form, then the converse also holds.

Third, we determine the hardness of determining causes (Thm.\ref{thm:comp}), showing that this is less complex than the general case for \ACu{} and \ACm{}, and leverage our classification of actual causality in fault trees to develop algorithms that find all causes for \ACo{} and \ACm{} (Sec. \ref{sec:algo}). Generalizing these algorithms to \ACu{} does not lead to an appreciable (theoretical) speedup over naïve algorithms; we leave the development of fast algorithms for \ACu{} to future work.

\hlbox{
Summarized, our contributions are: \vspace{-0.5em}
\begin{enumerate}
\item A translation from FTs to causal models, allowing us to apply the notions of actual causality to FTs (Def.~\ref{def:ftcm}).
\item For FTs, a classification of {\ACo}, {\ACu} and {\ACm} in terms of their graph structure and Boolean structure (Thms. \ref{thm:maino}, \ref{thm:mainu}, \ref{thm:mainm}).
\item We show that MCS elements in FTs become actual causes with the MCS as context (Thm.~\ref{thm:mcs}).
\item We show that, given a context, all MCSs that happened give rise to actual causes in \ACo{} and \ACu{} (Thm.~\ref{thm:mcs2}). This characterizes MCSs for tree-shaped FTs and FTs in disjunctive normal form.
\item We show that determining AC in FTs is NP-complete for \ACo{} and \ACm{}, and polynomial for \ACu{} (Thm. 10).
\item We give algorithms for finding AC in FTs for \ACo{} and \ACm{} (Sec. \ref{sec:algo}).
\item Proofs and illuminating examples for all these statements.
\end{enumerate}}

\section{Fault trees}\label{sec:fault-trees}

\GCmar{Define paths in FTs.}\MLZ{Not necessary, I would say, since FTs are DAGs.}


In a directed graph $(V,E)$, a \emph{root} or \emph{sink} is a vertex without outgoing edges; a \emph{leaf} (\emph{source}) is a vertex without ingoing edges.
The \emph{inputs} of a vertex $v$ are its predecessors: $\ch(v) = \{w \in V \mid (w,v) \in E\}$.





\myparagraph{Fault trees.} 
A FT is a systematic, graphical tool for analyzing why systems fail. 
FT analysis proceeds by breaking down the system-level failure into its causes, and these causes into subcauses, until the root causes are found. These are called \emph{basic events} (BEs), represented by the leaves of the tree; see Fig.~\ref{fig:FT-doorbell-expl}.
Higher-level failures are connected using an \emph{AND-gate} when all subcauses must occur for the failure to propagate, or by an \emph{OR-gate} when a single subcause is sufficient. Formally, a FT is a directed acyclic graph, whose non-leaf nodes are labeled with AND- and OR-gates.

\label{p:FT}

\begin{definition}
A \emph{FT} is a triple $T = (V,E,\gamma)$, where $(V,E)$ is a directed acyclic graph with a unique root $\R_{T}$, and $\gamma\colon V \rightarrow \{\tBE,\tOR,\tAND\}$ satisfies $\gamma(v) = \tBE$ iff $v$ is a source.
\end{definition}

\begin{wrapfigure}[9]{r}{0.4\linewidth}
\centering
\vspace{-3em}
\begin{tabular}{lll}
\hline
AC & Actual causality & p\pageref{p:AC} \\
\ACm{} & modified AC & p\pageref{p:ACm}\\
\ACo{} & original AC & p\pageref{p:ACo}\\
\ACu{} & updated AC & p\pageref{p:ACu}\\
BE & basic event & p\pageref{p:BE}\\
CM & causal model & p\pageref{p:CM}\\
FT & fault tree & p\pageref{p:FT}\\
MCS & minimal cut set & p\pageref{p:MCS}\\
\hline
\end{tabular}
\caption{Abbreviations used.}

\vspace{-3em}
\end{wrapfigure}

The FT is called \emph{tree-shaped} if the graph $(V,E)$ is a (directed) tree, i.e., no vertex has two successors. In other cases, we call it \emph{DAG-shaped.} 
The nodes of a FT are usually called {\em events}: the root $\R_T$ of the tree is called the {\em top level event}, the leaves are the {\em basic events} and all other nodes are {\em intermediate events}. 
We write $\BE{T} = \{v \in V \mid \gamma(v) = \tBE\}$ for the set of basic events. We omit the subscript if $T$ is clear from the context. \label{p:BE}

The FT in Fig.~\ref{fig:FT-doorbell-expl} features
three basic events {\fElectricity} ($\mathsf{EF}$), {\fPublic} ($\mathsf{PF}$) and {\fOperator} ($\mathsf{OF}$),
one intermediate event {\fAlert} ($\mathsf{AF}$),
and the top event {\fLock} ($\mathsf{LF}$). 

\myparagraph{Status vectors.}
A {\em status vector} indicates for each basic event its status, i.e., whether that BE has failed. Write $\BB := \{0,1\}$; a status vector for $T$ is
a  $\vec u \in \BB^{\BE{}}$, where $u_v=1$ indicates that BE $v$ has failed and 
$u_v=0$ that $v$ is operational. We often identify a status vector with its set of basic events: Given a set $C\subseteq \BE{}$, its
  status vector ${\vec u}^C$ is given by 
$u^C_{v}=1$ for $v\in C$ and $u^C_v=0$ for $v\not\in C$.
Given a $b \in \mathbb{B}$, we write
$\vec u[v\leftarrow b]$
for the vector that equals $\vec u$, except that it sets $v$ to $b$: $(\vec u[v\leftarrow b])_v=b$
and $(\vec u[v\leftarrow b])_{v'}=(\vec u)_{v'}$ for $v\neq v'$.


\myparagraph{Structure function.}
The semantics of a FT is defined by its structure function. Given a status vector $\vec u$ and a node $v$, $\struc{T}(\vec u,v)$ indicates whether the node $v$ fails on status vector $\vec u$.

\begin{definition}\label{def:struct-func}
The \emph{structure function} of a FT $T$ is a function $\struc{T}\colon\BB^{\BE{T}}\times V \rightarrow \BB$ defined recursively as
\begin{align*}
\struc{T}(\vec u,v) = \begin{cases}
\vec u_v & \textrm{ if $\gamma(v) = \tBE$}\\
\bigvee_{w \in \ch(v)} \struc{T}(\vec u,w) & \textrm{ if $\gamma(v) = \tOR$}\\
\bigwedge_{w \in \ch(v)} \struc{T}(\vec u,w) & \textrm{ if $\gamma(v) = \tAND$}
\end{cases}
\end{align*}
We abbreviate 
$\struc{T}(\vec u) =
\struc{T}(\vec u, \R_T)$ and $\struc{T}(C,v) = \struc{T}({\vec u}^C,v)$ for a set of basic events $C$.
Two FTs $T$ and $T'$ are {\em equivalent} if they have the same structure function, i.e. 
$\struc{T}(\vec u,v) =
\struc{T'}(\vec u,v)$ for all $\vec u,v$. 
\end{definition}
The structure function of the FT in Fig.~ \ref{fig:FT-doorbell-expl} is given by $\Phi(\vec{u}) = u_{\mathsf{EF}} \vee (u_{\mathsf{PF}} \wedge u_{\mathsf{OF}})$.
In particular, $\struc{}(0,1,0)=0$
and $\struc{}(1,1,0)=1$, also written as
 $\struc{}(\{\mathsf{PF}\})=0$
and $\struc{}(\{\mathsf{EF},\mathsf{PF}\})=1$.

\myparagraph{Cut sets.} A cut set is a set of basic events that causes the top event to fail; a cut set is minimal if no proper subset is a cut set. 
Minimal cut sets are a key tool in FT analysis, as they constitute  a minimal cause for the FT to fail. \label{p:MCS}

\begin{definition}
    A \emph{cut set} is a subset $C \subseteq \BE{T}$ such that $\struc{T}(C,\R) = \tone$. A \emph{minimal cut set} (MCS) furthermore satisfies $\struc{T}(C',\R) = \tzero$ for all $C' \subset C$.
\end{definition}
The following are cut sets of the FT in Fig.~\ref{fig:FT-doorbell-expl}; The first two sets are MCSs. 
\[
\{\electricity\}~~~~ \{\public, \operator\}~~~~ \{\electricity, \public\}~~~~\{\electricity, \operator\}~~~~\{\electricity, \public, \operator\}
\]

\myparagraph{Coherence.}
Since (standard) FTs do not contain any negations, they are {\em coherent}: a cut set remains a cut set if we add more elements, because a failed system remains failed if more failures occur. 
Formally, the structure function is monotonous: 
if $C\subset C'$, then $\struc{}(C,v)\leq 
\struc{}(C',v)$ for all events $v$.

A BE $e$ is critical for a status vector if flipping its value from $0$ to $1$ causes the system to transition from non-failure to failure.
\begin{definition}
 Basic event $e$ is {\em critical} for status vector $\vec{u}$ if 
 $\struc{}(\vec u) = 0$ and 
$\struc{}(\vec u[e\leftarrow 1]) = 1.$
A basic event $e$ is {\em relevant} if there is a vector in which $e$ is critical. Otherwise, it is {\em irrelevant}.
\end{definition}

A basic event $e$ is irrelevant iff the structure function is constant in $u_{e}$ iff $e$ is not an element of any MCS.

\myparagraph{Causality in FTs.}
In FT analysis, each MCS is seen as a potential root cause. All elements of a given MCS are seen as causes, since we need each of them to occur for a top level failure.

\section{Preliminaries on causal models}\label{sec:AC-prelim}

 Halpern and Pearl~\cite{HalpernP01}
 define actual causality (AC)\label{p:AC} in terms of 
 Causal Models (CMs), composed of variables and their corresponding values. The dependencies between variables are expressed through structural equations encoding the system's behavior. We first recap CMs and associated concepts as in~\cite{HalpernP01}.

\subsection{Causal models and networks} 

A \emph{causal model} is defined based on a signature $ {S}$ consisting of a pair $( {\cnodes},  {\inodes})$, where $ {\cnodes}$ is a set of \emph{exogenous variables}, $ {\inodes}$ a set of \emph{endogenous variables}.\label{p:CM}
For notational convenience, we assume all variables to be Boolean. 


Intuitively, exogenous variables represent a \emph{context} ($\cnodes$), or external factors: conditions outside the model’s control that influence the system, but are not explained by the model itself. Endogenous variables are the \emph{internal} variables ($\inodes$) of the model.
Take the fish doorbell example of the introduction.
For instance, we let $\ExElectricity \in \cnodes$
(``Electricity Failure'') 
denote an exogenous variable representing the external failure of supply of electrical power. Its value determines the endogenous variable $\electricity \in \inodes$, which captures the availability of electricity. Thus, $\ExElectricity$ serves as the contextual factor governing the state of $\electricity$, and may be thought of as reflecting the environment’s provision (or interruption) of power.
\begin{wrapfigure}[14]{r}{0.5\linewidth}
    \centering
\vspace{-2em}
\includegraphics[width=0.4\textwidth]{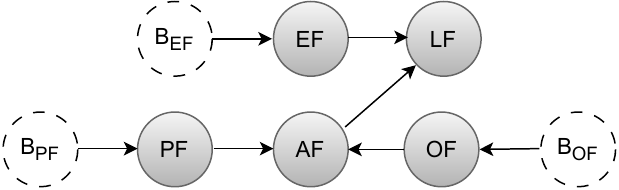}
\caption{Causal network for $M_{\doorbellcm}$ (see Example \ref{eg:CM-fish-doorbell}). Exogenous variables (dashed):
$\ExElectricity$ electricity availability,
$\ExPublic$ public's engagement,
$\ExOperator$ operator's availability.
Endogenous variables (grey):
$\electricity$ electrical supply failure,
$\public$ public's alert failure,
$\operator$ operator's alert failure,
$\alert$ combined alert failure, 
$\lock$ lock failure.
}
\label{fig:doorbell-causal-network}
\end{wrapfigure}
The values of endogenous variables $X \in \inodes$ are determined by structural equations $F_X$, based on the exogenous  and other endogenous variables.


\begin{definition}
A \emph{causal model}  over a signature $ {S} = ( {\cnodes},  {\inodes})$ is a tuple $M = ( {S},  {F})$, where $ {F}$ assigns to each endogenous variable $X \in  {\inodes}$ a function
\[
F_X\colon \BB^{\cnodes}\times\BB^{\inodes\setminus\{X\}}\rightarrow\BB.
\]
\end{definition}

Each $F_X$ captures how the value of $X$ is determined by the other variables in $ {\cnodes} \cup  {\inodes}$.  
The functions $F_X$ define a so-called set of \emph{structural equations}; causal models are also referred to as \emph{structural equation models (SEMs)}.

\begin{example}[CM ``fish doorbell'']\label{eg:CM-fish-doorbell}
The CM for the ``fish doorbell'' scenario captures how an electrical supply failure or an alerting failure can enable a failure of the lock.
The set of endogenous variables are the same events present in Fig.~\ref{fig:FT-doorbell-expl}: $\inodes = \{\mathsf{EF},\mathsf{PF},\mathsf{OF},\mathsf{AF},\mathsf{LF}\}$. We further introduce $\ExElectricity \in \cnodes$, representing the external supply of electrical power, which affects $\electricity$.
$\ExPublic \in \cnodes$ corresponds to the level of public engagement (whether there are enough online observers to notice fish) which influences $\public$.
Furthermore, $\ExOperator \in \cnodes$, representing the operator's availability, determines $\operator$. The corresponding CM $M_{\doorbellcm}$ and its structural equations follow the fault tree's gate logic and are given by:
\begin{align}
\electricity = \ExElectricity \quad \quad \public = \ExPublic \quad \quad \operator = \ExOperator \quad \quad \alert = \public \land \operator \quad \quad \lock = \electricity \lor \alert \nonumber
\end{align}
\end{example}
\myparagraph{Causal networks.}
A CM induces a \emph{causal network}: a directed edge from a
variable $Y \in \inodes \cup \cnodes$ to a
variable $X \in \inodes$ 
\MSmar{Conventions for int / ext/ all variables? X, U, etc?}
exists whenever the value of $F_X$ depends on the value of $Y$. 
Exogenous variables in $\cnodes$ have no incoming edges as they are the `inputs' of the network.
AC~\cite{HalpernP01} is typically restricted to \emph{acyclic} causal networks, ensuring that, for a given context, the structural equations have a unique solution. 

Causal networks provide an intuitive and informative means of understanding the causal pathways within a CM. The causal network corresponding to Example \ref{eg:CM-fish-doorbell} is provided in Fig.~\ref{fig:doorbell-causal-network}. 

\subsection{Reasoning in causal models}
Given a signature $S = ({\cnodes}, {\inodes})$, an expression $X = x$, where $X \in {\inodes}$ and $x \in \BB$, is referred to as a \emph{primitive event}. A \emph{basic causal formula} has the structure $[Y_1 \leftarrow y_1, \ldots, Y_k \leftarrow y_k] \varphi$, 
where $\varphi$ is a Boolean combination of primitive events, and each $Y_i \in \inodes$ is a distinct endogenous variable. This formula is abbreviated as $[\vec{Y} \leftarrow \vec{y}] \varphi$ and, intuitively, it states that $\varphi$ holds when $Y_i$ is set to $y_i$.

\myparagraph{Context and interventions.}
A \emph{context} refers to an assignment $\vec{u}$ of values to all exogenous variables in ${\cnodes}$. We denote by $(M, \vec{u})$ the CM $M$ under the specific assignment $\vec{u}$ for $\cnodes$.
An 
intervention modifies a CM $M = ({S}, {F})$ by fixing an endogenous\footnote{Interventions are often restricted to endogenous variables, as exogenous variables are considered given \cite{ActualCausalityHalpern}. Mathematically, this restriction is not necessary.} variable $X \in \inodes$ 
to a particular value $x$. 
Interventions are considered as external effects on the system, not governed by the structural equations. 
The result is a new model, written $M_{X \leftarrow x}$, which is identical to $M$ except that the structural equation  $F_X$
is replaced with the constant assignment $X = x$.
When applying multiple interventions, such as $X_1 \leftarrow x_1, \ldots, X_n \leftarrow x_n$, we write $\vec{X} \leftarrow \vec{x}$ and denote the modified model by $M_{\vec{X} \leftarrow \vec{x}}$.

\myparagraph{Satisfaction.} We write $(M, \vec{u}) \models \varphi$ to indicate that the causal formula\MSmar{Say that $\varphi$ is a boolean formula over endogenous variable}\GCmar{this is stated in the first paragraph of the section} $\varphi$ 
holds in the CM $M$ under context $\vec{u}$. The satisfaction relation $\models$ is defined inductively. Specifically, $(M, \vec{u}) \models X = x$ iff the variable $X$ takes the value $x$ in the unique solution to the structural equations of $M$ under context $\vec{u}$. Uniqueness is guaranteed since~\cite{HalpernP01} considers acyclic models.
Conjunctions and negations are interpreted in the standard logical way. Furthermore, $(M, \vec{u}) \models [\vec{Y} \leftarrow \vec{y}] \varphi$ holds iff $(M_{\vec{Y} \leftarrow \vec{y}}, \vec{u}) \models \varphi$.
Potential \emph{actual causes} are of the form $X_1 = x_1 \land \cdots \land X_k = x_k$, with $X_i \in \inodes$, which can be abbreviated as $\vec{X} = \vec{x}$.

\begin{example}[CM reasoning]\label{eg:models}    
Consider again the CM $M_{\doorbellcm}$ in Example~\ref{eg:CM-fish-doorbell}. A context $\vec u$ sets the exogenous variables. Assume $\vec u$ sets $\ExElectricity = 0$, $\ExPublic = 1$ and $\ExOperator = 0$: the public is not engaged, electricity is available and the operator is  available. Setting endogenous variables (e.g., $\alert \leftarrow 1$) corresponds to interventions that override those variables. The following hold:
\begin{align*}
(M_\doorbellcm, {\vec u}) &\models {\lock} = 0, &  (M_\doorbellcm, {\vec u})  &\models {\alert} = \wedge {\electricity} = 0,\\  
(M_\doorbellcm, {\vec u})  &\models [\alert \leftarrow 1]({\lock} = 1).
\end{align*}
In the last equation, the intervention $[\alert \leftarrow 1]$ changes the behavior of the lock from a non-failing one (${\lock} = 0$) to a failing one (${\lock} = 1$).
\end{example}

\section{Fault trees as causal models}\label{sec:FTasCM}


Following the examples above, we embed FTs in CMs as follows.

\begin{definition} \label{def:ftcm}
Let $T = (V,E,\gamma)$ be a FT.
Its CM $M_T$ is given by:
\begin{itemize}
    \item For the signature $S=(\cnodes,\inodes)$, 
    we associate  to each FT node $v$ an endogenous variable $X_v$, i.e., we set $\inodes= \{X_v \mid v\in V\}$. 
    Further, we associate to each basic event $e$ an exogenous variable $Y_e$, i.e., 
     $\cnodes= \{ Y_e \mid e\in \tBE\}$.

    \item The equations of $M_T$ follow the FT logic:
\begin{align*}
X_v = Y_v &\quad \textrm{ if $\gamma(v) = \tBE$}, & \quad \quad \quad \quad \\
X_v = \bigvee_{w \in \ch(v)} X_w &\quad \textrm{ if $\gamma(v) = \tOR$},\\
X_v = \bigwedge_{w \in \ch(v)} X_w &\quad\textrm{ if $\gamma(v) = \tAND$}.
\end{align*}
\end{itemize}
\end{definition}

\begin{wrapfigure}[5]{r}{0.3\linewidth}
    \centering
    \vspace{-12em}
\includegraphics[width=0.3\textwidth]{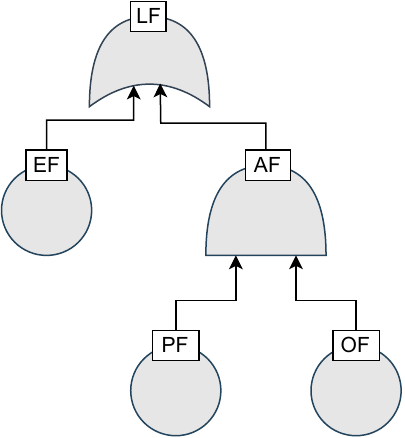}
\vspace{-0.75em}
\caption{Fish doorbell failure FT.}
\label{fig:FT-doorbell-copy}
\end{wrapfigure}

As basic events appear twice in $M_T$, 
     we could have taken $\cnodes=\varnothing$. 
     However, by considering basic events as endogenous variables, each context in $M_T$ corresponds to a status vector in $T$.
Further, we note that FT analysts are merely interested in causal reasoning about failure of the top event, i.e., $X_{\R} = 1$.

\begin{example}[CM from FT.]
The CM associated to the lock failure FT in Fig.~\ref{fig:FT-doorbell-copy} is given by:
\begin{align*}
\cnodes              &= \{Y_{\electricity}, Y_{\public}, Y_{\operator}\} & X_{\electricity} &= Y_{\electricity} & X_{\alert} &= X_{\public} \land X_{\operator}\\
\inodes                &= \{X_{\electricity}, X_{\public}, X_{\operator}, X_{\alert}, X_{\lock}\} & X_{\operator} &= Y_{\operator} & X_{\lock} &= X_{\electricity} \lor X_{\alert} \\
&& X_{\public} &= Y_{\public}
\end{align*}
\end{example}

This construction guarantees that, given a FT status vector $\vec u$, each variable in the CM gets assigned the correct value, i.e., the value given by the structure function applied to $\vec u$.
In other words, the structure function is the unique solution to the CM equations.

\begin{lemma}
Let $T$ be a FT and $M_T$ be the CM associated to $T$. 
Then for all FT nodes $v$ and all
context $\vec u$ we have:
$
    (M_T,\vec u)\models X_v = \struc{T}(\vec u,v).
$
\end{lemma}


The following result 
phrases the coherency of FTs
in terms of CMs. 
Intuitively, it states that setting a variable $X \leftarrow 0$ can only decrease the value of other variables, and setting $X \leftarrow 1$ can only increase the value of other variables. Its proof is straightforward induction.

\begin{theorem} \label{thm:mono} 
Let $T$ be a FT and let $M_T$ be its associated CM. Then for any context $\vec{u}$, and for any two 
variables $X,Y \in \inodes$ we have:
\begin{enumerate}
\item $(M_T,\vec{u}) \models Y = 0$ implies $(M_T,\vec{u}) \models [X \leftarrow 0]Y = 0$;
\item $(M_T,\vec{u}) \models [X \leftarrow 1]Y = 0$ implies $(M_T,\vec{u}) \models Y = 0$;
\item $(M_T,\vec{u}) \models Y = 1$ implies $(M_T,\vec{u}) \models [X \leftarrow 1]Y = 1$;
\item $(M_T,\vec{u}) \models [X \leftarrow 0]Y = 1$ implies $(M_T,\vec{u}) \models Y = 1$.
\end{enumerate}
\end{theorem}


\section{Actual causality}\label{sec:actual-causality}

Several notions of Halpern and Pearl's actual causality (AC) have been defined over time. In what follows, we present the original {\ACo} definition from~\cite{HalpernP01}, followed by two refinements: the updated {\ACu} and the modified {\ACm} definitions.
All three versions are included in Definition 2.2.1 of~\cite{ActualCausalityHalpern}.

\subsection{The original definition {\ACo}}\label{sec:ACo}

Each definition of actual causality consists of three conditions, typically referred to as $\AC{1}$, $\AC{2}$, and $\AC{3}$ that must hold for a conjunction of primitive events $\vec X = \vec x$ to be an actual cause for the effect $\varphi$ in a causal model $M$ under the context $\vec{u}$.
Intuitively, $\AC{1}$ states that both the cause $\vec X = \vec x$ and the effect $\varphi$ hold for  $M$ in context $\vec{u}$.
Note that $\AC{2}$ uses a partition $(\vec{Z}, \vec{W})$ of the endogenous variables $\inodes$.
$\AC{2}(a^o)$ represents the \emph{counterfactual test} under \emph{contingency} $\vec{W}$: if the cause were absent, the effect should also disappear, assuming the interventions in $\vec{W}$ that modify $M$. 
Contingencies are particularly useful when causes are redundant, as they allow testing whether the counterfactual still holds after disabling one potential cause. This isolates the causal contribution of each factor: $\vec{X} = \vec{x}$ becomes pivotal under $\vec{W}$.
The variables in $\vec Z$ can be viewed as constituting the \emph{causal path} to $\varphi$~\cite{ActualCausalityHalpern}: intuitively, intervening on some variable in $\vec X$ propagates changes through variables in $\vec Z$, ultimately altering the value of $\varphi$. Interventions in $\vec W$ may indirectly modify the values of variables in $\vec Z$. According to $\AC{2}(b^0)$, these changes must leave $\varphi$ unaffected, even when certain variables in $\vec Z$ are held at their original values. $\AC{2}(b^0)$ corresponds to a \emph{sufficiency} condition.
Finally, $\AC{3}$ expresses \emph{minimality}.
\label{p:ACo}
\begin{definition}[H\&P Actual Causality: \ACo]\label{def:AC-o}
A conjunction of primitive events 
    $\vec X = \vec x$ over a vector of internal variables $\vec X$
    is an {\em actual cause} of $\varphi$ in $({M},\vec u)$ if the following three conditions hold:
    \begin{description}\itemsep6pt
        \item[\AC{1}] $({M},\vec u)\models (\vec X = \vec x)$ and $({M},\vec u)\models \varphi$.
        \item[\AC{2}] There exists a partition $(\vec Z, \vec W)$ of $\inodes$ with $\vec X \subseteq \vec Z$, and a setting $\vec w$ of the variables in $W$ such that both of the following conditions hold:
        \begin{itemize}\itemsep6pt
            \item ($a^o$) There exists a setting $\vec{x}'$ of the variables in $\vec{X}$ such that\\
              $({M},\vec u)\models[\vec X \leftarrow \vec x', \vec W \leftarrow \vec w]\neg \varphi$;
\item ($b^o$) If $\vec{z}^*$ is such that $({M}, \vec{u}) \models \vec{Z} = \vec{z}^*$, then for all subsets $\vec{Z}'$ of $\vec{Z} - \vec{X}$, we have\\
              $({M},\vec u)\models[\vec X\leftarrow \vec x, \vec W \leftarrow \vec w, \vec Z'\leftarrow \vec z^*]\varphi$.
         \end{itemize}
        \item[\AC{3}] $\vec X$ is minimal: no proper subset of $\vec X$ satisfies conditions \AC{1} and \AC{2}.
    \end{description}
\end{definition}


Without loss of generality, in our examples, the endogenous variables $(\in \inodes)$ occurring in $\varphi$ are excluded from $Z$.

\begin{example}[{\ACo} causes]\label{eg:aco-cause}
Consider $M^-_\doorbellcm$; a variation of  $M_\doorbellcm$, without the variable $\alert$, but encoding the same behavior of lock failure, defined as (see Fig.~\ref{fig:doorbell-causal-network-simpl}):
\begin{align}
\electricity &= \ExElectricity; & \public &= \ExPublic; & \operator &= \ExOperator; & \lock &= \electricity \lor (\public \land \operator). \label{eq:doorbellcm-}
\end{align}
 Assume the context $\vec u = (1,1,0)$, i.e., it sets $\ExElectricity = 1$, $\ExPublic = 1$ and $\ExOperator = 0$.
 We will show that $\public = 1$ (public's alert failure) is an actual cause of the effect $\varphi$ defined as $\lock = 1$ (lock failure) in $(M^-_\doorbellcm, \vec{u})$.
\end{example}

\begin{wrapfigure}[7]{r}{0.4\linewidth}
\centering
\vspace{-3.5em}
\includegraphics[width=\linewidth]{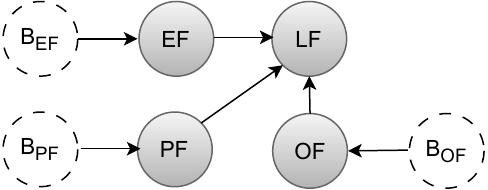}
\caption{Causal network for $M^-_{\doorbellcm}$.}
\label{fig:doorbell-causal-network-simpl}
\end{wrapfigure}
$\AC{1}$ requires that both the cause $\vec X = \vec x$ and the effect $\varphi$ hold in the actual world $\vec u$.
$\AC{1}$ is satisfied, as:
\begin{align}\label{eq:ACO-AC1}
(M^-_\doorbellcm, \vec{u}) &\models \public = 1,\\
(M^-_\doorbellcm, \vec{u}) &\models \lock = 1. \nonumber
\end{align}

$\AC{2}(a^o)$ corresponds to a \emph{necessity condition}. Namely, setting the causal variables $\vec X$ to different values $\vec x'$ changes the satisfiability of $\varphi$ from true to false. $\AC{2}(a^o)$ is sometimes referred to as the \emph{counterfactual test} and $W$ as the \emph{contingent variables}.

In $(M^-_\doorbellcm, \vec{u})$, consider $\vec{W} = \{\electricity, \operator\}$ and the setting $\{\electricity \leftarrow 0$ and $\operator \leftarrow 1\}$ (i.e., $\vec w = (0,1)$).
In this setting, the value of $\public$ is pivotal:
Let $x'$ be the setting $\public \leftarrow 0$ (i.e., $\vec x' = (0)$). 
Then ${\AC{2}(a^o)}$ and  ${\AC{2}(b^o)}$ hold 
(note $Z' = \varnothing$):
\begin{equation}\label{eq:ACO-AC2a}
(M^-_\doorbellcm, \vec{u}) \models [\public \leftarrow 0, \electricity \leftarrow 0, \operator \leftarrow 1]  (\lock = 0)
\end{equation}
While
\begin{equation}\label{eq:ACO-AC2b}
         (M^-_\doorbellcm, \vec{u})  \models [\public \leftarrow 1, \electricity \leftarrow 0, \operator \leftarrow 1](\lock = 1) 
\end{equation}

$\AC{2}(b^o)$ corresponds to a \emph{sufficiency condition}.
It guarantees that it is $\vec{X} \leftarrow \vec{x}$ (acting through the directed paths involving the variables in $\vec{Z}$) that is genuinely responsible for $\varphi$.

$\AC{3}$ is a minimality condition that ensures the identified set of causal variables $\vec X$ is the smallest set witnessing $\varphi$ in $(M^-_\doorbellcm, \vec u)$.
$\AC{3}$ trivially holds in our example, as the singleton set $\{\public\}$ satisfies \AC{1}-\AC{2}. Hence, also based on (\ref{eq:ACO-AC1}), (\ref{eq:ACO-AC2a}) and (\ref{eq:ACO-AC2b}), it follows that $\public = 1$ is an actual cause of $\lock = 1$. \vspace{1em}

Now consider the CM $M_{\doorbellcm}$ from Example \ref{eg:CM-fish-doorbell}, which includes the additional variable $\alert = \public \land \operator$ (see Fig.~\ref{fig:doorbell-causal-network}).
$M_{\doorbellcm}$ and $M^-_{\doorbellcm}$ are logically equivalent.
However, in $M_{\doorbellcm}$, $\public = 1$ is \emph{not} an actual cause of $\lock = 1$ in the given context $\vec u = (1,1,0)$. 
The reason is that changing $\public$ from 0 to 1 in setting $\vec{W}$ is no longer pivotal.
There are two cases to consider:
If ${\alert} \in \vec{W}$, then
$\lock$ will take the value of $\alert$, 
irrespective of $\public$.
(In fact, $\alert=0$ is an actual cause 
for $\lock = 0$).
Otherwise, if $\alert \in \vec{Z}$, with $\alert = 0$ in the actual world, $\AC{2}(b)^0$ again fails. The choice of variables (what is treated as primitive and what is derived) crucially affects what qualifies as actual cause.

Interestingly, actual causes under {\ACo} are always singletons.

\begin{theorem}\emph{\cite[Thm.~2.2.3(d)]{ActualCausalityHalpern}} \label{thm:singleton}
    If $\vec X = \vec x$ is an actual cause of $\varphi$ in $(M,\vec u)$ according to Definition~\ref{def:AC-o}), then $|\vec X| = 1$.
\end{theorem}

\subsection{The updated definition {\ACu}}

In the CM $M^-_\doorbellcm$ of Example~\ref{eg:aco-cause}, $\public$ is identified as a cause of $\lock$. However, this may be regarded as a partial form of causality, since it results in a lock failure ($\lock = 1$) only in conjunction with an operator failure ($\operator = 1$) that did not occur. To better capture such cases, a so-called \emph{updated actual causality} {\ACu} was proposed in~\cite{ActualCausalityHalpern}.\label{p:ACu}
 {\ACu} redefines {$\AC{2}(b^o)$} in Definition~\ref{def:AC-o} to:\\[1ex]
{$\AC{2}(b^u)$} For all subsets $\vec{W}' \subseteq \vec{W}$ and $\vec{Z}' \subseteq \vec{Z}-\vec{X}$, we have
\begin{equation}\label{eq:ACU}
({M}, \vec{u}) \models [ \vec{X} \leftarrow \vec{x}, \vec{W}' \leftarrow \vec{w}, \vec{Z}' \leftarrow \vec{z}^* ]\varphi
\end{equation}
where $\vec{z}^*$ denotes the actual values of the variables in $\vec{Z}'$.

\begin{example}[\ACu~Causes]\label{eg:ACU}
    According to {\ACu}, $\public = 1$ is no longer a cause of $\lock = 1$ in $M^-_\doorbellcm$~(\ref{eq:doorbellcm-}). Recall that for the counterfactual test to hold, we need to consider ${\electricity} \in \vec{W}$. Furthermore, the original value of $\operator$ is $0$. This leads to the violation of~(\ref{eq:ACU}) when taking $\vec{Z}' = \varnothing$ and $\vec{W}' = \{\electricity\}$, since $(M^-_\doorbellcm,\vec u) \models [\public \leftarrow 1, \electricity \leftarrow 0](\lock = 0)$.
\end{example}

Unlike {\ACo}, {\ACu} allows for multiple element causes \cite{ActualCausalityHalpern}. It also subsumes {\ACo}:
\begin{theorem} \label{thm:acu-aco} \emph{\cite[Thm.~2.2.3(c)]{ActualCausalityHalpern}}
    If $ X =  x$ is part of an actual cause of $\varphi$ in $({M},\vec u)$ according to the updated definition {\ACu}, then $ X =  x$ is an actual cause according to the original definition {\ACo}.
\end{theorem}

\subsection{The modified definition {\ACm}}

To make the definition of causality more straightforward, a simplified form of {\AC{2}} was proposed~\cite{ActualCausalityHalpern}.
The corresponding notion of causality is referred to as the {\emph{modified actual causality}} {\ACm}.
The core idea is to restrict attention to variable assignments that correspond to the actual situation we are analysing. Rather than considering hypothetical values for the contingency variables in ${\vec{W}}$, the modified definition only looks at the specific values these variables take in the considered context $\vec{u}$. An important consequence of this simplification is that it makes $\AC{2}(b^o)$ and ${\AC{2}}(b^u)$ subsumed by the modified {\AC{2}$^m$} below, whenever $\vec{W}$ is held at its actual values. As a result, there is no longer a need for these additional conditions, since the only contingencies we need to consider are those that align with the actual context. Hence, in the modified definition {\ACm}, condition \label{p:ACm} {\AC{2}} simply becomes:\\[1ex]
{\AC{2}}$^m$ There is a set $\vec{W} \subseteq \inodes$ and a setting $\vec{x}'$ of the variables in $\vec{X}$ such that 
\begin{equation}\label{eq:ACM}
({M}, \vec{u}) \models [\vec{X} \leftarrow \vec{x}', \, \vec{W} \leftarrow \vec{w}^*] \lnot \varphi
\end{equation}
with $\vec{w}^*$ the values of the variables in $\vec{W}$ in the actual context.

Note that {\ACm} subsumes both {\ACo} and {\ACu}:
\begin{theorem} \label{thm:acum-acu-aco}
\emph{\cite[Thm.~2.2.3(a+b)]{ActualCausalityHalpern}} The following hold:
\begin{itemize}
     \item[(a)]
    If $ X =  x$ is part of an actual cause of $\varphi$ in $({M},\vec u)$ according to the modified definition {\ACm}, then $ X =  x$ is an actual cause according to the original definition {\ACo}.
    \item[(b)]
    If $ X =  x$ is part of an actual cause of $\varphi$ in $({M},\vec u)$ according to the modified definition {\ACm}, then $ X =  x$ is an actual cause according to the updated definition {\ACu}.
\end{itemize}
\end{theorem}

The complexity of computing ACs for each of the definitions {\ACo}, {\ACu} and {\ACm} is analysed in~\cite{ActualCausalityHalpern,AC-complexity}. The results show that computing AC is hard; e.g., {\ACo} is NP-complete  in binary models (where all variables are binary). In binary models, SAT solving has been used to compute {\ACm} \cite{pretschner2019efficient}.

\section{Actual causality in fault trees} \label{sec:ACinFT}

Our main results are the classifications of {\ACo}, {\ACu}, and {\ACm} for FTs: Given a FT $T$ and a context $\vec{u}$ for its causal model $M_T$ (so $\vec{u}$ is a status vector), we give necessary and sufficient conditions for any $\vec{X} = \vec{x}$ to be an actual cause for $X_{\R} = 1$. We only consider $X_{\R} = 1$ because our results transfer to any variable $X_v = 1$ by considering the sub-FT rooted at $v$; and we are more interested in what causes failure than in what causes non-failure. The results for $X_{\R} = 0$ can be obtained by applying our results to the dual FT \cite{dhillon1981engineering}.

\subsection{\ACo{} in FTs} \label{ssec:ACo}

The classification of \ACo{} for FTs is as follows:

\begin{theorem}[Classification of \ACo{} for FTs] \label{thm:maino}
Let $T$ be a FT, let $M_T$ be its causal model, and let $\vec{u}$ be a context. Let $\vec{X}$ be a set of variables in $M_T$, and let $\vec{x}$ be a possible value for $\vec{X}$. Then $\vec{X} = \vec{x}$ 
satisfies {\ACo} for $X_{\R}=1$ iff:
\begin{enumerate}
    \item There exists a node $v \in V$ such that $\vec{X} = \{X_v\}$, and $\vec{x} = 1$;
    \item There is a path $v = v_0,\ldots,v_n = \R$ 
in the directed graph $(V,E)$ such that:
\begin{enumerate}
    \item 
    $(M_T,\vec{u}) \models X_{v_i} = 1$
    for all $0 \leq i \leq n$, 
    \item For all $i,j > 0$, if $\ch(v_i) \cap \ch(v_j) \neq \varnothing$, then $\gamma(v_i) = \gamma(v_j)$.
\end{enumerate}
\end{enumerate}
\end{theorem}

Condition 1 requires that causes are singletons (by Theorem \ref{thm:singleton}), and that only events $X_v = 1$ can cause $X_{\R} = 1$. The latter is a consequence of Theorem \ref{thm:mono}, by the lack of negation in FTs: an event $X_v = 0$ will never \emph{cause} a later variable to be equal to 1. 



\begin{example} \label{ex:aco}
Consider again the lock failure FT, with context $\vec{u} = Y_{\mathsf{EF}}Y_{\mathsf{PF}}Y_{\mathsf{OF}} = (1,1,0)$. We picture this FT again as $T_1$ in Fig.~\ref{fig:aco-example}, with the resulting variable values inscribed in each gate. The unique path from $\mathsf{EF}$ to the root $\mathsf{LF}$ satisfies condition 2 from Theorem \ref{thm:maino}: all values are 1, and condition 2(b) holds because there are no shared inputs. It follows that $X_{\mathsf{EF}} = 1$ is an actual cause for $X_{\mathsf{LF}} = 1$ under \ACo{}. By contrast, $X_{\mathsf{PF}} = 1$ does not satisfy \ACo{} for $X_{\mathsf{LF}} = 1$, because the unique path $\mathsf{PF} \rightarrow \mathsf{LF}$ contains a $0$. FT $T_2$ has the same structure function as $T_1$; here, however, the unique path $\mathsf{PF} \rightarrow \mathsf{LF}$ contains only 1s, so here $X_{\mathsf{PF}} = 1$ does satisfy \ACo{} for $X_{\mathsf{LF}} = 1$.
\end{example}

\begin{wrapfigure}[18]{r}{0.3\linewidth}
    \centering
    \includegraphics[width=0.8\linewidth]{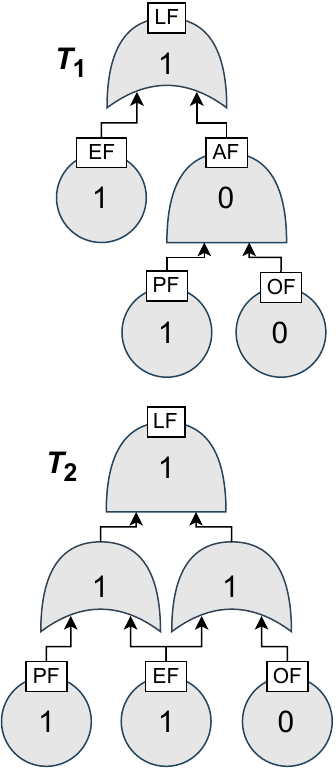}
    \caption{Two FTs with the same structure function ($T_1 =$ Fig.~\ref{fig:FT-doorbell-expl}). With context $\vec{u} = Y_{\mathsf{EF}}Y_{\mathsf{PF}}Y_{\mathsf{OF}} = (1,1,0)$, $X_{\mathsf{PF}} = 1$ satisfies \ACo{} for $X_{\mathsf{LF}} = 1$ in $T_1$, but not in $T_2$.}
    \label{fig:aco-example}
\end{wrapfigure}

Example \ref{ex:aco} points to an interesting property of causality in FTs. The two FTs of Fig.~\ref{fig:aco-example} have identical structure functions, hence identical minimal cut sets; thus in FT analysis they would be considered equivalent. However, under the context $\vec{u} = (1,1,0)$ we see that $X_{\mathsf{EF}}=1$ is a cause of system failure in one FT, but not in the other: \emph{\textbf{Equivalent FTs may behave differently as causal models}}. From the causality perspective, this is a feature, not a bug: 
Endogenous variables in CMs are not simply functions of their inputs, but represent real-world events that can be interacted with independently of their inputs. For FTs, this may capture a form of the epistemic uncertainty that we might have modelled our system incorrectly: the possibility of an intervention setting an intermediate gate represents the possibility that an intermediate gate might fail despite its children not failing, because in modelling our system we missed a failure cause.

\begin{wrapfigure}[19]{r}{0.3\linewidth}
\centering
\vspace{5.5em}
\includegraphics[width=4cm]{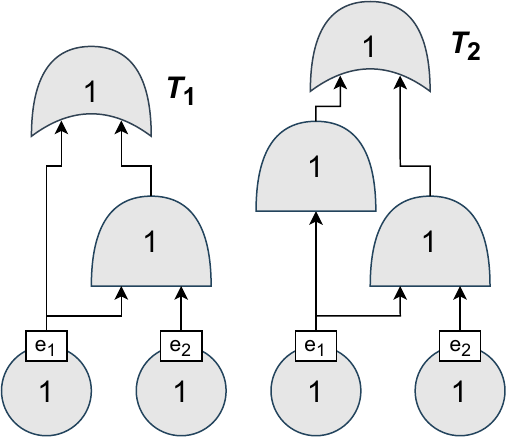}
\caption{$X_{e_2} = 1$ satisfies \ACo{} for $X_{\R}=1$ in $T_2$, but not in $T_1$.} \label{fig:irrelevant}
\end{wrapfigure}
Theorem \ref{thm:maino}.2b is there to exclude some pathological cases.This condition does not hold, for excample, in $T_1$ of Fig.~\ref{fig:irrelevant}: the path $e_2 \rightarrow \R$ has two inputs that share the input $e_1$, but have different labels. Therefore, $X_{e_2} =1$ is not an actual cause (for \ACo{}) for $X_{\R} = 1$. 
This is understood from the structure function $X_{e_1} \vee (X_{e_1} \wedge X_{e_2}) \equiv X_{e_1}$. The BE $e_2$ is irrelevant as it does not appear in the structure function; this is reflected by the fact that it cannot be an actual cause. Indeed, $e_2$ does not satisfy Theorem \ref{thm:maino}.

Precisely because such a construction makes events irrelevant, one rarely sees paths in which AND-gates and OR-gates share children: Of the 54 real-world FTs in the benchmark of \cite{basgoze2022bdds,garagiola2024coyan,lopuhaa2025fault}, only 10 have any such paths. Thus Theorem \ref{thm:maino}.2b is mathematically necessary, but plays a limited role in practice: Having a path $v \rightarrow \R$ of only 1-valued nodes is often enough to imply \ACo{}. Note that it is also possible to construct FTs in which events that are irrelevant from the FT perspective can still be actual causes, such as in $T_2$ of Fig.~\ref{fig:irrelevant}. The key here is again that changing the graph structure impacts how causality behaves, even if it leaves the structure function unchanged.

\subsection{\ACu{} in FTs}

The classification of \ACu{} for FTs is similar to that of \ACo{}. As before, we only consider causes of $X_{\R} = 1$. Unlike in general causal models \cite[Ex.~2.8.2]{ActualCausalityHalpern}, the result below states that \ACu{} in FTs only admits singleton causes.

\begin{theorem}[Classication of \ACu{} for FTs] \label{thm:mainu}
Let $T$ be a FT, let $M_T$ be its causal model, and let $\vec{u}$ be a context. Let $\vec{X}$ be a set of variables in $M_T$, and let $\vec{x}$ be a possible value for $\vec{X}$. Then $\vec{X} = \vec{x}$ 
satisfies {\ACu} for $X_{\R}=1$ if and only if the following holds:
\begin{enumerate}
    \item There exists a node $v \in V$ such that $\vec{X} = \{X_v\}$, and $\vec{x} = 1$;
    \item There is a path $v = v_0,\ldots,v_n = \R$ 
in the directed graph $(V,E)$ such that:
\begin{enumerate}
    \item 
    $(M_T,\vec{u}) \models X_{v_i} = 1$
    for all $0 \leq i \leq n$, 
    \item For all $i,j > 0$, if $\ch(v_i) \cap \ch(v_j) \neq \varnothing$, then $\gamma(v_i) = \gamma(v_j)$.
    \item $(M_T,\vec{u}) \models [\vec{D} \leftarrow \vec{0}]X_{\R}=1$, where \[\vec{D} = \bigcup_{i>0\colon \gamma(v_i) = \tOR}(\ch(v_i)\setminus\{v_{i-1}\}).\] 
\end{enumerate}
\end{enumerate}
\end{theorem}


Unlike in general causal models \cite[Ex.~2.8.2]{ActualCausalityHalpern}, Theorem \ref{thm:mainu} states that \ACu{} in FTs only admits singleton causes. Once again, the culprit is the coherency of Theorem \ref{thm:mono}: the difference between \ACo{} and \ACu{} is that eq. \eqref{eq:ACU} must hold for all subsets $\vec{W}' \subseteq \vec{W}$, rather than just $\vec{W}$ itself. However, due to coherency, in FTs it turns out that it just has to hold for the worst-case $\vec{W}'$ consisting only of those variables that have value $0$ in $\vec{w}$. The fact that eq. \eqref{eq:ACU} only has to hold for one $\vec{W}'$ aligns the definition of \ACu{} closer to \ACo{}, and this turns out to be sufficient for the proof of Theorem \ref{thm:singleton} to also hold for \ACu{} in FTs. Other than that, the difference between \ACo{} and \ACu{} for FTs lies in condition 2c. 

\begin{figure}[t]
    \centering
    \includegraphics[width=0.7\linewidth]{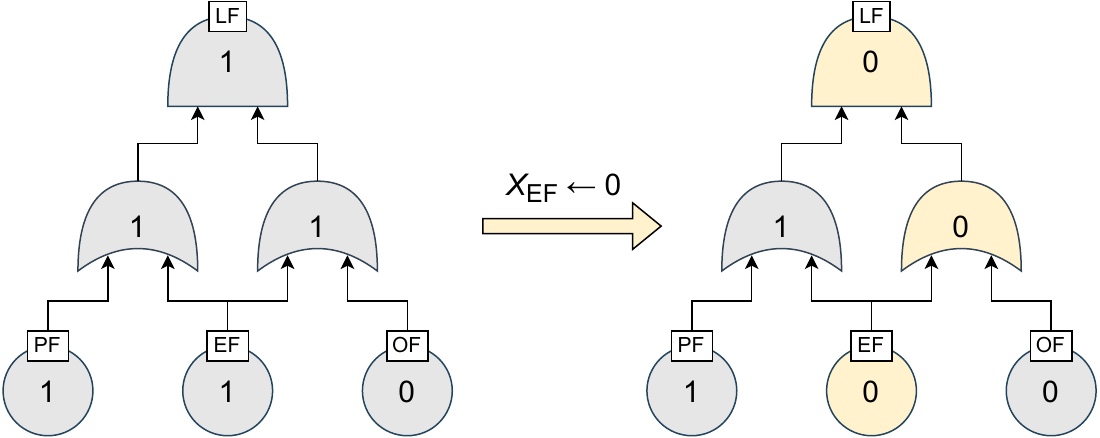}
    \caption{$T_2$ of Fig.~\ref{fig:aco-example} with context $\vec{u} =  (1,1,0)$ and intervention $X_{\mathsf{EF}} \leftarrow 0$.}
    \label{fig:ACu_example}
\end{figure}

\begin{example} \label{ex:ACu}
Consider the FT of Fig.~\ref{fig:ACu_example} with context vector $\vec{u} = Y_{\mathsf{EF}}Y_{\mathsf{PF}}Y_{\mathsf{OF}} = (1,1,0)$. Then $X_{\mathsf{PF}} = 1$ satisfies Theorem \ref{thm:maino}, so $X_{\mathsf{PF}} = 1$ is a cause of $X_{\mathsf{LF}} = 1$ under \ACo{}. However, the path $\mathsf{PF} \rightarrow \R$ does not satisfy Theorem \ref{thm:mainu}.2c: if we set $X_{\mathsf{EF}} = 0$ by intervention, $X_{\mathsf{LF}}$ becomes $0$. It follows that $X_{\mathsf{PF}} = 1$ is \emph{not} an actual cause of $X_{\mathsf{LF}} = 1$ under \ACu{}.
\end{example}

Example \ref{ex:ACu} illustrates how \ACu{} brings causality in FTs closer to the structure function: Under \ACu{}, this FT now behaves identically to the FT of Fig.~\ref{fig:FT-doorbell-copy} to which it is equivalent. However, it is still possible for two equivalent FTs to behave differently under \ACu{}: in $T_1$ of Fig.~\ref{fig:irrelevant} with $\vec{u} = (1,1)$, the event $X_{e_2} = 1$ is not a cause of $X_{\R} = 1$, but it is a cause in the equivalent FT $T_2$.

\subsection{\ACm{} in FTs} \label{ssec:acm}

For \ACm{}, we get the following classification. Interestingly, it looks completely different to those of \ACo{} and \ACu{}. The most striking difference is that it does not refer to the graphical structure of the FT; instead \ACm{} can be characterised purely in terms of the structure function. 

We only consider causes that consist of BEs. This is for notational convenience; one can always consider non-BE causes by changing these nodes into BEs and removing their inputs.

\begin{theorem}[Classication of \ACm{} for FTs] \label{thm:mainm}
Let $T$ be a FT, let $M_T$ be its causal model, and let $\vec{u}$ be a context. Let $\vec{X}$ be a set of variables in $M_T$, and let $\vec{x}$ be a possible value for $\vec{X}$. Suppose the variables in $\vec{X}$ represent a set $C \subseteq V$ such that $C \subseteq \BE{T}$. Then $\vec{X} = \vec{x}$ satisfies {\ACm} for $X_{\R} = 1$ if and only if the following are satisfied:
\begin{enumerate}
\item $(M_T,\vec{u}) \models X_{\R}=  1$;
\item $\vec{x} = \vec{1}$, and $u_v = 1$ for all $v \in C$;
\item If $D = \{v \in \BE{T} \mid u_v = 1\}$, then $\struc{T}(D \setminus C,\R) = 0$;
\item $C$ is minimal w.r.t. property 3.
\end{enumerate}
\end{theorem}


\begin{example}
Consider again the ``fish doorbell'' FT in Fig.~\ref{fig:FT-doorbell-copy}, with $\vec{u} = Y_{\mathsf{EF}}Y_{\mathsf{PF}}Y_{\mathsf{OF}} = (1,1,1)$; this corresponds to the set of BEs $D = \{\mathsf{EF},\mathsf{PF},\mathsf{OF}\}$. Consider $\vec{X} = (Y_{\mathsf{PF}},Y_{\mathsf{OF}})$ and $\vec{x} = (1,1)$; thus $\vec{X}$ corresponds to the set of BEs $C = \{\mathsf{PF},\mathsf{OF}\}$. Clearly, $\vec{X} = \vec{x}$ satisfies conditions 1,2 of Theorem \ref{thm:mainm}. Furthermore, $\struc{T}(D\setminus C,\R) = \struc{T}(\{\mathsf{EF}\},\R) = 0$, so 3 is satisfied. Finally, one can check that taking $C = \{\mathsf{PF}\}$ or $C = \{\mathsf{OF}\}$ does not satisfy 3, so 4 is satisfied as well. Hence $\vec{X} = \vec{x}$ is an actual cause of $X_{\mathsf{AF}} = 1$ under \ACm{}.
\end{example}

The fact that \ACm{} so closely reflects the structure function also allows us to classify MCSs in terms of \ACm{}:

\begin{corollary} \label{cor:mcs}
Let $T$ be an FT and lets $M_T$ be its causal model. A subset $D \subseteq \BE{T}$ is an MCS iff the following hold:
\begin{enumerate}
\item $\struc{T}(D,\R) = 1$;
\item The only variable sets in the causal model $M_T$, under context $\vec{u}^D$, that satisfy {\ACm} for $X_{\R} = 1$ are singletons.
\end{enumerate}
\end{corollary}


\section{From minimal cut sets to actual causality} \label{sec:mcs}

It may be striking that our classification of actual causality in FTs of Theorems \ref{thm:maino}, \ref{thm:mainu} \& \ref{thm:mainm} has nothing to do with MCSs, as these are understood to be potential causes of system failures. In this section, we show that being part of an MCS is a necessary, but generally not sufficient, condition for being an actual cause.

\subsection{Fault trees without given context}

First, we consider FTs without a given context $\vec{u}$. In this setting, the pertinent question is whether it is possible for a BE to ever be an actual cause. It turns out that being a \emph{relevant} BE is a sufficient condition; for \ACm{} it is also necessary. Recall that a BE $v$ is relevant if it is part of some MCS, or equivalently, if the structure function is non-constant in $X_v$.

\begin{theorem} \label{thm:mcs}
Let $T$ be a FT, and let $v$ be a basic event. 
\begin{enumerate}
    \item If there is an MCS $D$ such that $v \in D$, then $\vec{u}^D$ is a context under which $X_v = 1$ satisfies {\ACo}, {\ACu}, and {\ACm} for $X_{\R} = 1$. 
    \item If there exists a context under which $X_v = 1$ satisfies \ACm{} for $X_{\R} = 1$, then $v$ is relevant.
\end{enumerate}
\end{theorem} 


Theorem \ref{thm:mcs}.2 does not hold for \ACo{} and \ACu{}: In $T_2$ of Fig.~\ref{fig:irrelevant} $e_2$ is not relevant, but it satisfies \ACo{} and \ACu{} for $\vec{u} = (1,1)$.

\subsection{Fault trees with given context}

Now we consider FTs with a given context $\vec{u} = \vec{u}^D$. The question is now to what extent MCS $C$ that have happened, i.e. that satisfy $C \subseteq D$, give rise to actual causes. For \ACm{}, this does not have a neat description: in fact, Theorem \ref{thm:mainm}.3 states that \ACm{} is best expressed in terms of maximal sets of events that still result in nonfailure, rather than minimal sets of events that result in failure. For \ACo{} and \ACu{}, however, we find that for every MCS $C 
\subseteq D$, every $v \in C$ is an actual cause of top event failure. This sufficient condition is not necessary: In $T_2$ of Fig.~\ref{fig:aco-example}, $D = \{\mathsf{EF},\mathsf{OF}\}$, and $\mathsf{OF}$ is not part of any MCS present in $C$; however, $X_{\mathsf{OF}} = 1$ still satisfies \ACo{} for $X_{\R} = 1$. To state when the condition is necessary, we recap the definition of FTs in disjunctive normal form (DNF). A DNF FT consists of a top OR-gate, a second layer of AND-gates, and a third layer of BEs; the ANDs represent its MCSs.

\begin{definition}
An FT $T$ is of \emph{disjunctive normal form (DNF)} if:
\begin{enumerate}
\item $\gamma(\R_T) = \tOR$
\item $\gamma(v) = \tAND$ for $v \in \ch(\R_T)$
\item $\gamma(w) = \tBE$ for all $w \in \bigcup_{v \in \ch(\R_T)} \ch(v)$
\item For each $v \neq v' \in \ch(\R_T)$, one has $\ch(v) \not \subseteq \ch(v')$.
\end{enumerate}
\end{definition}

We then get the following result:

\begin{theorem} \label{thm:mcs2}
Let $T$ be a FT, and let $\vec{u}$ be a context of $M_T$. Let $v$ be a BE. Consider the following statements:
\begin{enumerate}
\item $\{v \in \BE{T} \mid u_v = 1\}$ contains an MCS containing $v$;
\item $X_v = 1$ satisfies {\ACu} for $X_{\R} = 1$;
\item $X_v = 1$ satisfies {\ACo} for $X_{\R} = 1$.
\end{enumerate}
Then 1 implies 2, and 2 implies 3. If $T$ is tree-shaped or if $T$ is of disjunctive normal form, then all three are equivalent.
\end{theorem}




\section{Computational complexity}

Determining whether something is a cause in binary models in general is NP-complete for \ACo{}/\ACm{} \cite{ActualCausalityHalpern} and so-called $\textrm{D}_2^{\textrm{P}}$\emph{-complete} (worse than NPC) for \ACu{} \cite{AC-complexity}. For FTs, we show that this is less complex for \ACu{} and \ACm{}:

\begin{problem} \label{prob:dec}
Given a FT $T$, a status vector $\vec{u}$, an event $w$, and set of events $\vec{X}$, determine if $\vec{X} = \vec{1}$ is actual cause of $X_w = 1$ under {\ACo}/{\ACu}/{\ACm}.
\end{problem}

\begin{theorem} \label{thm:comp}
Problem \ref{prob:dec} NP-complete for \ACo{} and \ACu{}, and solvable in polynomial time for \ACm{}. 
\end{theorem}

A full computational study, where these complexity results are accompanied by algorithms that scale well in practice, is left for future work.

\section{Algorithms} \label{sec:algo}

\begin{algorithm}[t]
\SetAlgoLined
\KwData{Fault tree $T = (V,E,\gamma)$, status vector $\vec{u}$}
\KwResult{Set of all $v \in V$ for which $X_v = 1$ satisfies \ACo{} for $X_{\R} = 1$}
Compute $P_{T,\vec{u}}$ and $\mathcal{F}(P_{T,\vec{u}})$\;
\eIf{$\R \notin P_{T,\vec{u}}$}
{
\Return $\varnothing$
}
{
\eIf{there is an $l \in V$ such that $(\R,l) \in \mathcal{F}(P_{T,\vec{u}})$}
{
$\mathcal{S}(\R) \leftarrow \{\{\R\}\}$\;
}
{
$\mathcal{S}(\R) \leftarrow \{\varnothing\}$\;
}
\For{$v \in V_{P_{T,\vec{u}}}$ in reverse topological order}
{
$\mathcal{S}(v) \leftarrow \varnothing$\;
\For{$w \in \outp(v)$}
{
\For{$H \in \mathcal{S}(w)$}
{
\If{there is no $h \in H$ such that $(h,v) \in \mathcal{F}(P_{T,\vec{u}})$}
{\eIf{there is a $l \in V$ such that $(v,l) \in \mathcal{F}(P_{T,\vec{u}})$}
{
$\mathcal{S}(v) \leftarrow \mathcal{S}(v) \cup \{H \cup \{v\}\}$\;
}
{
$\mathcal{S}(v) \leftarrow \mathcal{S}(v) \cup \{H\}$\;
}
}
}
}
}
\Return $\{\R\} \cup \{v \in \ch(w) \mid w \in V\colon \mathcal{S}(w) \neq \varnothing\}$
}
\caption{Computing \ACo{} for fault trees} \label{alg:aco}
\end{algorithm}

In previous sections we discussed how actual causes in FTs are characterized mathematically. In this section, we discuss how these characterizations can be used algorithmically: we describe algorithms for finding actual causes under \ACo{}, \ACu{}, and \ACm{}. 

\subsection{Algorithms for \ACo{}}

Using Theorem \ref{thm:maino}, we can construct a naive algorithm that answer whether $X_v =1$ satisfies \ACo{} for $X_{\R} = 1$, by first checking whether $(M_T,\vec{u}) \models (X_v = 1) \wedge (X_{\R} = 1)$, and then considering all paths $v \rightarrow \R$ and checking whether one of them satisfies properties (2a) and (2b). The fundamental issue with this approach is that the number of paths between two vertices in a DAG with $n$ vertices is $\mathcal{O}(2^n)$, hence this approach scales badly for large $T$. Theorem \ref{thm:comp} shows that deciding \ACo{} in FTs is NP-complete, hence worst-case exponential behaviour is likely to be unavoidable; nevertheless, we can improve on this naive approach.

More precisely, one may note that the impediment for a path to show \ACo{} is the presence of an AND-gate and an OR-gate on it with a shared child. As we discuss in Section \ref{ssec:ACo}, this typically leads to irrelevant BEs in the FT, and for this reason such pairs of gates rarely occur in real-world FTs. Therefore, it makes sense to find the set of such \emph{forbidden pairs}, i.e., the set
\[
\mathcal{F}(T) = \left\{(h,l) \in V^2 \ \middle|\ \substack{\exists \text{ path } l \rightarrow h \textrm{ and } \gamma(l) \neq \gamma(h)\\ \textrm{and } \ch(h) \cap \ch(l) \neq \varnothing}\right\}.
\]
Note that each forbidden pair consists of a `high' element $h$ closer to $
\R$, and a `low' element $l$ (see Fig.~\ref{fig:algexample}).

The set $\mathcal{F}(T)$ can be used as follows. If one were just interested in nodes $v$ that have a path $v \rightarrow \R$ containing only 1s, a simple algorithm would suffice: starting from a set  $Q = \{\R\}$ of vertices from which $\R$ can be reached, we keep adding to $Q$ vertices that evaluate to 1 under $\vec{u}$ and that have an edge into $Q$, until no more vertices can be added. This finds all relevant vertices in linear time. 

For our problem, we only want such vertices whose paths to $\R$ do not contain two elements of a forbidden pair. That means that instead of storing binary information at each vertex (part of $Q$ yes/no), we must store which `high' elements we have encountered so far, to signal that we cannot add the corresponding `low' elements. Since every node can have multiple paths to the root containing different high elements, we need to store a set of sets of high elements at each vertex.

This idea is the basis for Algorithm \ref{alg:aco}, which finds \emph{all} vertices $v$ such that $X_v = 1$ satisfies \ACo{} for $X_{\R} = 1$. Since we only care about nodes that evaluate to 1, we formalize this as the \emph{positive part} of $T$ (line 1 of Algorithm \ref{alg:aco}):

\begin{definition}
Let $T$ be a FT, and let $\vec{u}$ a status vector of $T$. Then the \emph{positive part of $T$ under $\vec{u}$} is the sub-FT $P_{T,\vec{u}}$ of $T$ of all $v$ such that $\struc{T}(\vec{u},v) = 1$.
\end{definition}

Of course, if $X_{\R} = 0$, then $X_{\R} = 1$ cannot have any cause as it is not true (line 3). Otherwise, we will define, for each node $v$, a set of sets of high elements $\mathcal{S}(v)$, where $H \in \mathcal{S}(v)$ means that there is a path $v \rightarrow \R$ which contains precisely the high elements in $H$, and none of their corresponding low elements. Thus, we initialize $\mathcal{S}(\R)$ as either $\{\{\R\}\}$ or $\{\varnothing\}$, depending on whether $\R$ is itself a high element in a forbidden pair or not (lines 5-9). Each other $\mathcal{S}(v)$ is initialized as $\varnothing$ (line 11), since we have not found a path to $\R$ yet.

Next, we address nodes one by one, starting at $\R$ and working our way down (line 10). At $v$, we look at all $H$ in all $\mathcal{S}(w)$. If $v$ is not the low element of any high element in $H$ (line 14), then we add either $H$ or $H \cup \{v\}$ to $\mathcal{S}(v)$, depending on whether $v$ is a high element itself or not (lines 15-19).

At the end, we return the set of \emph{inputs} of all vertices $w$ with $\mathcal{S}(w) \neq \varnothing$. This is because if $v \in \ch(w)$ with $\mathcal{S}(v) = \varnothing$ and $\mathcal{S}(w) \neq \varnothing$, this means that $v$ is a low element corresponding to a high element in every set in $\mathcal{S}(w)$. However, by Theorem \ref{thm:maino}, the first element of a path is allowed to be part of a forbidden pair.

Essentially, this algorithm visits every edge once, so for fixed $k = |\mathcal{F}(P_{T,\vec{u}})|$ the complexity is $\mathcal{O}(|E|)$ (see Remark \ref{rmk:ptu} for why we use $\mathcal{F}(P_{T,\vec{u}})$ rather than $\mathcal{F}(T)$ here and in Algorithm \ref{alg:aco}). The information transmitted over each edge is an element of $\mathcal{S}(w) \in \mathcal{P}(\mathcal{P}(\mathcal{F}(P_{T,\vec{u}})))$, which has size $\leq 2^k$ ($\mathcal{P}$ stands for powerset). 

To find all forbidden pairs, we have to look at each of the $\mathcal{O}(|E|^2)$ pair of edges to see whether they have the same source and their targets have different gate types, and to see whether there is a path between one of the targets and the other ($\mathcal{O}(|V|)$ steps per pair). As a result, we get the following theorem:

\begin{theorem}\label{thm:algo}
Let $k = |\mathcal{F}(P_{T,\vec{u}})|$. Algorithm \ref{alg:aco} finds all actual causes of $X_{\R} = 1$ with time complexity $\mathcal{O}(|E|^2|V| + 2^k|E|)$.
\end{theorem}

\begin{remark} \label{rmk:ptu}
In Algorithm \ref{alg:aco}, we could use $\mathcal{F}(T)$ rather than $\mathcal{F}(P_{T,\vec{u}})$, provided that we add some checks to see whether nodes evaluate to 1. Using $\mathcal{F}(P_{T,\vec{u}})$ works as well: The only relevant forbidden pairs (i.e., those that could potentially be part of a path proving \ACo{}) are of the form $(h,l)$ with both $h$ and $l$ evaluating to 1. Since one of these is an AND-gate, this means that their shared child must also evaluate to 1; thus $(h,l) \in \mathcal{F}(P_{T,\vec{u}})$. In other words, it is enough to just consider the forbidden pair of $P_{T,\vec{u}}$ in Algorithm \ref{alg:aco}; we chose to use this to not clutter the algorithm with checks for evaluation to 1. With regards to complexity, we have 
\[
|\mathcal{F}(T)| = |\mathcal{F}(P_{T,\vec{1}})| = \max_{\vec{u}} |\mathcal{F}(P_{T,\vec{u}})|,
\]
so using $P_{T,\vec{u}}$ could lead to better performance if $\vec{u}$ has fewer elements that evaluate to 1.
\end{remark}

\begin{example}
We consider the FT of Fig.~\ref{fig:algexample}, and apply Algorithm \ref{alg:aco}. First, only $c,e$ evaluate to 0, so $P_{T,\vec{u}} = T-\{c,e\}$ and we no longer consider these nodes. Removing these BEs does not change the set of forbidden pairs, so $\mathcal{F}(P_{T,\vec{u}}) = \{(i,g),(j,h)\}$. Next, we compute $\mathcal{S}(v)$ top-down, for each $v$:
\begin{itemize}
\item Since $\R = k$ is not in a forbidden pair: $\mathcal{S}(k) = \{\varnothing\}$.
\item At $i$, we only consider $H = \varnothing$ in line 13. This does not contain any pairs of which $i$ is the low element (line 14). However, $i$ is itself a high element (line 15), so we get $\mathcal{S}(i) = \{\{i\}\}$.
\item At $g$, we only consider $H = \{i\}$. Since $(i,g) \in \mathcal{F}(P_{T,\vec{u}})$, we do not add this $H$ to $\mathcal{S}(g)$, and $\mathcal{S}(g) = \varnothing$.
\item At $a$, there are no $H$ to consider, so $\mathcal{S}(a) = \varnothing$.
\item Analogously we get $\mathcal{S}(j) = \{\{j\}\}$ and $\mathcal{S}(h) = \mathcal{S}(d) = \varnothing$.
\item At $f$, we consider $H = \{i\}$ and $H = \{j\}$. The node $f$ is not a low element corresponding to these high elements, and $f$ is not a high element itself, so $\mathcal{S}(f) = \{\{i\},\{j\}\}$. We also get $\mathcal{S}(b) = \{\{i\},\{j\}\}$.
\end{itemize}
Finally, the actual causes are $k$ and inputs of nodes with nonzero $\mathcal{S}$: $i,j,g,h,f,b$.
\end{example}

\subsection{Algorithms for \ACu{}}

\begin{wrapfigure}[17]{r}{5cm}
    \centering
    \vspace{-5.7em}
    \includegraphics[width=\linewidth]{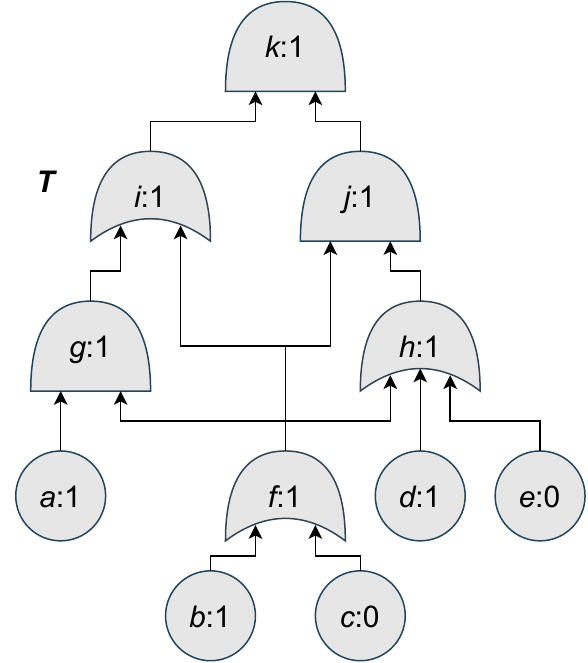}
    \caption{An example FT; node names and structure function values are inscribed. We have $\mathcal{F}(T) = \{(i,g),(j,h)\}$, where $i,j$ are high elements and $g,h$ are low elements.}
    \label{fig:algexample}
\end{wrapfigure}

Unfortunately, an approach such as Algorithm \ref{alg:aco} appears not to work: by Theorem \ref{thm:mainu}, we also need to be able to set the children of all OR-gates on the path to 0 without affecting $\R$. Thus, at every node we cannot just store a set of sets of high elements; instead, we need a set of sets of [high elements and OR-gates]. Thus the stored information is exponential in $|V|$. At that rate, one might as well just check all paths, of which there are also $\mathcal{O}(2^{|V|})$ many. Including some polynomial factors to account for the checks that need to be done at each path, this leads to the following, somewhat unsatisfying, conclusion:

\begin{theorem} \label{thm:algu}
There exists an algorithm that finds all causes of $X_{\R} = 1$ under \ACu{} in time complexity $\mathcal{O}(2^{|V|}|V|^2|E|^2)$.
\end{theorem}

Finding a more efficient algorithm to compute \ACu{}, that circumvents the need to enumerate all paths, is left for future work.

\subsection{Algorithms for \ACm{}}

\begin{figure}
\centering
\includegraphics[width=10cm]{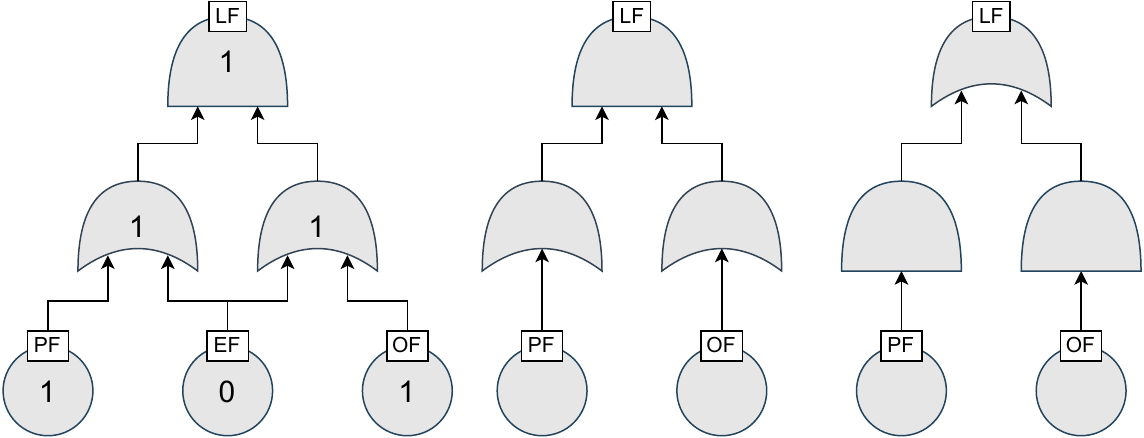}
\caption{From left to right: a fault tree $T$ with context $\vec{u} = (1,0,1)$; its positive part $P_{T,\vec{u}}$; the dual of its positive part $\check{P}_{T,\vec{u}}$.} \label{fig:MPS}
\end{figure}

As in Section \ref{ssec:acm}, we only look at causes consisting of basic events, mostly for notational convenience. Theorem \ref{thm:comp} was concerned with checking whether a given set of variables is a cause; now our task is to find all actual causes. this brings its own complications, since we are now looking for sets of events rather than just single events. At the same time, the fact that the problem is now purely about Boolean functions, rather than the graph structure, allows us to make use of existing methods for fault trees. In order to explain this, we first introduce the notion of \emph{minimal path sets}, which are minimal sets of basic events whose nonfailure ensures continued system functioning:

\begin{definition}
Let $T$ be a FT. A \emph{path set} is a set of basic events $C \subseteq \BE{T}$ such that $\struc{T}(\BE{T}\setminus C,\R) = 0$. A \emph{minimal path set} (MPS) furthermore satisfies $\struc{T}(\BE{T}\setminus C,\R) = 1$ for all $C' \subset C$.
\end{definition}

MPS is the dual notion of MCS: the MPSs of $T$ are precisely the MCSs of the dual FT $\check{T}$ obtained by exchanging all AND-gates into OR-gates and vice versa (see Fig.~\ref{fig:MPS}). With this terminology and the notion of the positive part $P_{T,\vec{u}}$, the following follows directly from Theorem \ref{thm:mainm}:

\begin{lemma} \label{lem:mps}
Let $T$ be a FT, and let $\vec{u}$ be a status vector of $T$. Let $\vec{X}$ be a set of variables in $M_T$, corresponding to a set of BEs $C \subseteq \BE{T}$, and let $\vec{x}$ be a possible value for $\vec{X}$. Then $\vec{X} = \vec{x}$ satisfies \ACm{} for $X_{\R} = 1$ if and only if the following are satisfied:
\begin{enumerate}
\item $(M_T,\vec{u}) \models X_{\R}=  1$;
\item $\vec{x} = \vec{1}$, and $u_v = 1$ for all $v \in C$;
\item $C$ is a minimal path set in $P_{T,\vec{u}}$.
\end{enumerate}
\end{lemma}

Thus, finding all actual causes of $X_{\R} = 1$ comes down to finding all MPS of $P_{T,\vec{u}}$, which in turn comes down on finding all MCSs of $\check{P}_{T,\vec{u}}$. Finding all MCSs is a core problem in FT analysis; an overview of existing approaches is given in \cite{ruijters2015fault}. These are typically worst-case exponential in the size of the FT, and differ in how well they perform on `typical' FTs.

\begin{theorem} \label{thm:algm}
There exists an algorithm that finds all actual causes consisting of BEs of $X_{\R} = 1$ with time complexity $\mathcal{O}(2^{|\BE{T}|}|E|)$.
\end{theorem}

\begin{example}
Consider the FT $T$ from Fig.~\ref{fig:MPS}, which is $T_2$ of Fig.~\ref{fig:aco-example}, with the context $\vec{u} = (1,0,1)$. Since $X_{\mathsf{LF}} = 1$, we can look for causes for this event. We get its positive part $P_{T,\vec{u}}$ by removing its 0-valued nodes, which is just $\mathsf{EF}$. To find the minimal path sets of this FT, we consider its dual $\hat{P}_{T,\vec{u}}$. This has structure function $X_{\mathsf{PF}} \vee X_{\mathsf{OF}}$, so its MCSs are $\{\mathsf{PF}\}$ and $\{\mathsf{OF}\}$. These are then also the MPSs of $P_{T,\vec{u}}$. We conclude that under this context, $X_{\mathsf{PF}} = 1$ and $X_{\mathsf{OF}} = 1$ are the actual causes under \ACm{} of $X_{\mathsf{LF}} = 1$. 
\end{example}

\section{Related Work}\label{sec:rel-work}
Causal reasoning has gained significant traction in recent years, particularly in fields such as artificial intelligence, science, and engineering. Many approaches build on counterfactual notions of causality inspired by Lewis~\cite{Lewis1973-LEWC-2}. As in~\cite{Lewis1973-LEWC-2}, determining whether $A$ caused $B$, involves evaluating a hypothetical scenario in which $A$ did not occur and observing whether $B$ would still take place. However, this counterfactual approach is often too weak \GCmar{@Milan: we are missing something here}
to capture causality in robust systems where multiple failures are needed to cause system failure.

To address this issue, Halpern and Pearl introduced a more nuanced framework~\cite{HalpernP01} based on structural models and Boolean equations. These, along with a set of formal conditions (commonly referred to as the AC conditions), characterise a cause of an outcome. This AC definition was later refined to account for richer causal scenarios~\cite{ActualCausalityHalpern}. These are particularly relevant to our investigation into the relationship between AC and causality in FTs.

Various other contributions build on AC to define liability frameworks for assessing responsibility, blame~\cite{DBLP:conf/sum/Halpern11}, and harm~\cite{DBLP:journals/mima/BeckersCH24}. Frameworks such as~\cite{Fenton-Glynn2017-FENAPP-2}, for instance, illustrate how AC can naturally be extended to address probabilistic causation as well. The work in~\cite{DBLP:conf/uai/Pearl12} formalises interventions (and counterfactuals) based on the so-called do-operator. This captures interventions by modifying the structural model: it overrides specific functional dependencies with fixed values, while leaving the remainder of the model unaffected.

Beyond the philosophical community, AC has been connected in various ways to models of computation in more technical fields.  In~\cite{DBLP:conf/vmcai/Leitner-FischerL13,DBLP:conf/tase/BonsangueCFT22,DBLP:journals/fuin/CaltaisMS20},
for instance, the AC framework is applied to transition systems and trace-based models used to represent concurrent system behaviour. The work in~\cite{DBLP:journals/fmsd/BeerBCOT12} follows AC to define the causes responsible for the specification failure observed in counterexample traces. In~\cite{DBLP:conf/icse/DubslaffWBA22}, the authors propose an AC-inspired notion of causality in configurable systems, to determine those features (e.g., execution time) that cause a given undesired system behaviour. In~\cite{TNO-causality}, the do-operator~\cite{DBLP:conf/uai/Pearl12} is used to define counterfactual-based tests for systems modeled as Bayesian Networks (BNs). This helped with speeding up the selection of the next diagnostic step in the root cause analysis of such systems. 
Moreover, BNs generalise FTs and can also be interpreted as structural causal models. However, computing conditional probabilities for formulas involving interventions, in the style of AC, presents significant challenges. Recent work, such as~\cite{DBLP:conf/nips/GalhotraH24}, explores a range of assumptions to enable the estimation of probabilities for interventional formulas in BNs.

In~\cite{DBLP:conf/gramsec/IbrahimRSAP20}, 
causal models are extracted from attack trees~\cite{mauw2005foundations} to support causal inference. While that work focuses on translating attack trees to CMs, we go beyond that by giving a full characterisation of AC as in \cite{ActualCausalityHalpern} within FTs, and investigating the relation between AC and minimal cut sets, the main ``causal'' tool of FT analysis.
An AC-based approach to extracting FTs encoding system failures was proposed in~\cite{DBLP:journals/ijccbs/Leitner-FischerL13}. These trees are, in essence, a more compact way of encoding a set of counterexample traces witnessing the violation of a specification. In~\cite{DBLP:journals/ijccbs/Leitner-FischerL13}, such traces are considered causal based on an adoption of AC to the setting of state-based systems.

\section{Conclusions}\label{sec:conclusions}
\MSmar{Nicely phrased!}
\MLZ{To do: rephrase conclusions to match introduction and rest of the paper.}
This paper offers a systematic study of fault trees from the perspective of the Halpern-Pearl framework of actual causality. We translate FTs to causal models, and classify actual causes in these models under the \ACo{}, \ACu{} and \ACm{} definitions. This shows how in FTs, \emph{causality} arises as a combination of their Boolean nature and their graph structure. Using the different definitions, introduced in AC to capture multi-element causes (\ACu{}) and to simplify assessments (\ACm{}), gives us degrees to emphasize the role of the graph structure, and deepens our understanding of system failures.

We also study the relation between minimal cut sets and causality: MCSs are not causes, but being an element of an MCS is a sufficient condition for being a cause. Because of the role of the graph structure, it is only a necessary condition when the FT satisfies specific assumptions.


Our results lay the groundwork for integrating causal techniques like interventions and counterfactuals into FT analysis, enhancing diagnosis and explanation within a unified framework for reliability and failure analysis.

Our study opens several promising research directions. As an immediate next step, we will explore algorithmic methods for computing actual causes in FTs, focusing on the complexity of {{\ACo}, {\ACu} and {\ACm} in practical cases.
Second, a natural extension is to integrate probabilistic reasoning. Since Bayesian Networks generalise FTs and link closely to CMs, we aim to investigate how probabilistic fault analysis can benefit from our causality-based approach.
Finally, our work paves the way for design-time causality analysis. Moving beyond post hoc diagnosis, we plan to explore how causality and responsibility can guide design decisions to reduce failures. 

\bibliographystyle{splncs04}
\bibliography{biblio}

\appendix

\section{Proofs}

In this appendix we prove the mathematical results of this paper. Before moving to the proofs proper, we reformulate causal models as \emph{graph causal models}, the notation of which will be more suited for the mathematical proofs.

\subsection{Graph causal models}

For proving our results, it will be convenient to reformulate the definition of CM, as intricate mathematical reasoning can be tricky in the notation proposed in the main text. For example, for two variables $X$ and $Y$ the statement $X=Y$ could mean that $X$ and $Y$ have the same value; that $F_X(Y) = Y$; or that $X$ and $Y$ are the same element of the set $\inodes \cup \cnodes$. We therefore give a new definition which takes the underlying graph as the primordial object.

\begin{definition}
A \emph{graph causal model} is a tuple $G = (\cV,\iV,\E,\mathcal{R},F)$ consisting of:
\begin{itemize}
\item Two disjoint sets $\cV$ and $\iV$, whose union is denoted $\V$;
\item A set $\E \subset \V \times \V$ such that $(\V,\E)$ is a directed acyclic graph, and every node in $\cV$ is a source;
\item A set $\mathcal{R}(v)$ for each $v \in \V$;
\item A function $F_v\colon \prod_{w \in \ch(v)}\mathcal{R}(w) \rightarrow \mathcal{R}(v)$ for each $v \in \iV$.
\end{itemize}
\end{definition}

Each $v \in \V$ is to be thought of as corresponding to a variable $X_v$, which takes values in $\mathcal{R}(v)$ (in the main text, we take $\mathcal{R}(v) = \BB$ throughout; we call such models \emph{binary graph causal models}). The tuple of variables corresponding to a subset $W \subseteq \V$ is denoted $X_W$, which takes values of the form $\mathcal{R}(W) := \prod_{w \in W} \mathcal{R}(w)$. Thus $X_W = x_W$ denotes the fact that for each $w \in W$, the variable $X_w$ takes the value $x_w$. Thus a \emph{context} is an element $\vec{u} \in \mathcal{R}(\cV)$ (for consistency with the main text, we keep contexts as $\vec{u}$; its coefficients are written $u_v$). 

To a graph causal model we associate a causal model as follows:

\begin{definition}
Let $G = (\cV,\iV,\E,\mathcal{R},F)$ be a binary graph causal model. Then its associated causal model $M_G = (\cnodes,\inodes,F)$ is defined as follows:
\begin{align*}
\cnodes &= \{X_v \mid v \in \cV\},\\
\inodes &= \{X_v \mid v \in \iV\},\\
F_{X_v}(\vec{Y}) &= F_v(X_{\ch(v)}),
\end{align*}
where $\vec{Y}$ are all variables in $\cnodes \cup \inodes \setminus \{X_v\}$.
\end{definition}


The transformation from graph causal models to causal models allows us to define the entire causality machinery on the level of graph causal models. For instance, if $G = (\cV,\iV,\E,\mathcal{R},F)$ is a graph causal model, $W \subseteq \iV$ is a set of internal nodes and $x_W \in \mathcal{R}(W)$, then we define the intervention $G_{X_W \leftarrow x_w}$ to be the graph causal model $({\iV}',{\cV}',\E',\mathcal{R}',F')$ defined as
\begin{align*}
{\iV}' &= \iV,\\
{\cV}' &= \cV,\\
\E' &= \E \setminus (\V \times W),\\
\mathcal{R}'(v) &= \mathcal{R}(v) \textrm{ for all $v \in \V$},\\
F'_v &\equiv x_v \textrm{ if $v \in W$,}\\
F'_v &= F_v \textrm{ otherwise.}
\end{align*}
It is straightforward to check that this is the same as intervention on causal models:

\begin{lemma} \label{lem:obv1}
Let $G = (\cV,\iV,\E,\mathcal{R},F)$ be a binary graph causal model, let $W \subseteq \iV$ and let $x_w \in \mathcal{R}(W) = \BB^W$. Then $M_{G_{X_W \leftarrow x_w}} = (M_G)_{X_W \leftarrow x_w}$. \qed
\end{lemma}

Furthermore, it will be convenient to express $(M,\vec{u}) \models \varphi$ as a \emph{structure function}, which we define as follows:

\begin{definition}
Let $G = (\cV,\iV,\E,\mathcal{R},F)$ be a graph causal model, and let $\vec{u} \in \mathcal{R}(\cV)$. Then for each $v \in \V$, define a value $\struc{G}(\vec{u},v) \in \mathcal{R}(v)$ recursively as follows:
\begin{itemize}
\item If $v \in \cV$, then $\struc{G}(\vec{u},v) = u_v$;
\item If $v \in \iV$, then let $x_w = \struc{G}(\vec{u},w)$ for all $w \in \ch(v)$. Then define $\struc{G}(\vec{u},v) = F_v(x_{\ch(v)})$.
\end{itemize}
\end{definition}

Because $(\V,\E)$ is a directed acyclic graph, this defines each $\struc{G}(\vec{u},v)$ uniquely. Intuitively, $\struc{G}(\vec{u},v)$ is the value $X_v$ obtains in the causal model $M_G$ under context $\vec{u}$:

\begin{lemma} \label{lem:obv2}
Let $G = (\cV,\iV,\E,\mathcal{R},F)$ be a binary graph causal model. Then for all $v \in V$ and all $\vec{u} \in \mathcal{R}(\cV) = \BB^{\cV}$ we have $(M_G,\vec{u}) \models X_v = \struc{G}(\vec{u},v)$. \qed
\end{lemma}

We write $\struc{G}(\vec{u},v)[X_W \leftarrow x_W]$ for $\struc{G[X_W \leftarrow x_w]}(\vec{u},v)$. By Lemmas \ref{lem:obv1} and \ref{lem:obv2}, we have (in a somewhat confusing notation)
\[
(M_G,\vec{u}) \models [X_W \leftarrow x_w]\Big(X_v = \struc{G}(\vec{u},v)[X_W\leftarrow x_w]\Big).
\]

\subsection{Fault trees as graph causal models}

If $T = (V,E,\gamma)$ is a FT, then we define the associated graph causal model $G_T = (\cV,\iV,\E,\mathcal{R},F)$ as
\begin{align*}
\cV &= \{\tilde{v} \mid v \in \BE{}\} \textrm{ is a disjoint copy of $\BE{}$},\\
\iV &= V,\\
\E &= E \cup \{(\tilde{v},v) \mid v \in \BE{}\},\\
\mathcal{R}(v) &= \BB \textrm{ for all $v \in \iV$},\\
F_v &= \begin{cases}
X_{\tilde{v}}, & \textrm{ if $\gamma(v) = \tBE$},\\
\bigvee_{w \in \ch(v)} X_{w}, & \textrm{ if $\gamma(v) = \tOR$},\\
\bigwedge_{w \in \ch(v)} X_{w}, & \textrm{ if $\gamma(v) = \tAND$}.
\end{cases}
\end{align*}

This is related to the causal model of a FT as follows:

\begin{lemma}
Let $T$ be a FT. Then $M_T = M_{G_T}$, once we identify $Y_v$ and $X_{\tilde{v}}$ for each BE $v$.
\end{lemma}

In this notation, we get a more straightforward version of theorem \ref{thm:mono}:

\begin{theorem} \label{thm:mono_app}
Let $T$ be a FT, let $G_T$ be its graph causal model, let $\vec{u} \in \BB^{\cV}$ be a context, and let $v,w \in \iV$. Then
\[
\struc{G_T}(\vec{u},v)[X_w \leftarrow 0] \leq \struc{G_T}(\vec{u},v) \leq \struc{G_T}(\vec{u},v)[X_w \leftarrow 1].
\]
\end{theorem}

\begin{proof}
Since all $F_v$ are nondecreasing, a straightforward proof by induction on $z$ shows that the function 
\begin{align*}
\BB &\rightarrow \BB \\
x &\mapsto \struc{G_T}(\vec{u},v)[X_w \leftarrow x]
\end{align*}
is nondecreasing in $x$. This fact, together with the fact that $\struc{G_T}(\vec{u},v)[X_w \leftarrow \struc{G_T}(\vec{u},v)] = \struc{G_T}(\vec{u},v)$, proves the theorem.
\end{proof}

\begin{proof}[Theorem \ref{thm:mono}]
Statements 1 and 4 are a direct consequence of the first inequality of Theorem \ref{thm:mono_app}, and statements 2 and 3 follow from the second inequality.
\end{proof}

\subsection{Actual causality for graph causal models}

We define \ACo{}, \ACu{} and \ACm{} for graph causal models; these are the exact counterparts of the definitions in Section \ref{sec:ACo}.

\begin{definition} \label{def:AC_app}
Let $G$ be a graph causal model, and let $\vec{u}$ be a context. Let $C\subseteq \iV$ be nonempty, let $x_C \in \mathcal{R}(C)$, let $v \in \iV$ and let $x_v \in \mathcal{R}(v)$. 
\begin{enumerate}
    \item $\X_C = x_C$ is said to satisfy \ACo{} for $\X_v = x_v$ if the following hold:
\begin{enumerate}
\item $\struc{G}(\vec{u},w) = x_w$ for all $w \in C \cup \{v\}$.
\item There exists a partition $\iV = Z \sqcup W$ such that $C \subseteq Z$, and values $y_W \in \mathcal{R}(W)$ and $y_C \in R(C)$, such that:
\begin{enumerate}
\item $\struc{G}(\vec{u},v)[\X_C \leftarrow y_C,\X_W \leftarrow y_W] \neq x_v$;
\item For any $z \in Z$, let $x_z^* = \struc{G}(\vec{u},z)$. Then 
\[\struc{G}(\vec{u},v)[\X_{Z'}\leftarrow x^*_{Z'},\X_W \leftarrow y_W] = x_v\]
for all $C \subseteq Z' \subseteq Z$.
\end{enumerate}
\item The set $C$ is minimal with respect to the first two properties.
\end{enumerate}
\item $\X_C = x_C$ is said to satisfy \ACu{} for $\X_v = x_v$ if the following hold:
\begin{enumerate}
\item $\struc{G}(\vec{u},w) = x_w$ for all $w \in C \cup \{v\}$.
\item There exists a partition $\iV = Z \sqcup W$ such that $C \subseteq Z$, and values $y_W \in \mathcal{R}(W)$ and $y_C \in \mathcal{R}(C)$, such that:
\begin{enumerate}
\item $\struc{G}(\vec{u},v)[\X_C \leftarrow y_C,\X_W \leftarrow y_W] \neq x_v$;
\item For any $z \in Z$, let $x_z^* = \struc{G}(\vec{u},z)$. Then 
\[\struc{G}(\vec{u},v)[\X_{Z'}\leftarrow x^*_{Z'},\X_{W'} \leftarrow y_{W'}] = x_v\]
for all $C \subseteq Z' \subseteq Z$ and $W' \subseteq W$.
\end{enumerate}
\item The set $C$ is minimal with respect to the first two properties.
\end{enumerate}
\item $\X_C = x_C$ is said to satisfy \ACm{} for $X_v = x_v$ iff the following hold:
\begin{enumerate}
\item $\struc{G}(\vec{u},w) = x_w$ for all $w \in C \cup \{v\}$.
\item For any $w \in \iV$, let $x_w^* = \struc{G}(\vec{u},w)$. Then there exists a subset $W \subseteq \iV$ and $y_C \in \mathcal{R}(C)$ such that
\[
\struc{G}(\vec{u},v)[X_C\leftarrow y_C,X_W\leftarrow x_W^*] \neq x_v.
\]
\item The set $C$ is minimal with respect to the first two properties.
\end{enumerate}
\end{enumerate}
\end{definition}

It is straightforward to check that these are precisely the analoga of AC on causal models:

\begin{lemma} \label{lem:obv3}
Let $G$ be a binary graph causal model, let $v$ be an internal node of $G$, and let $C$ be a set of internal nodes; let $x_C \in \mathcal{R}(C) = \BB^C$ and let $x_v \in \mathcal{R}(v) = \BB$. Then $X_C = x_C$ satisfies \ACo{}, \ACu{} or \ACm{} for $X_v = x_v$ in $G$ if and only if it does so in $M_G$.
\end{lemma}

Likewise, the following result can be proven completely analogously to Theorem \ref{thm:singleton} (see \cite[Thm. 2.2.3(d)]{ActualCausalityHalpern}):

\begin{proposition}\label{prop:singleton}
If $X_C = x_C$ satisfies \ACo{} for $X_v = x_v$, then $|C| = 1$.
\end{proposition}

\subsection{Proof of Theorem \ref{thm:maino}}

Before we fully classify \ACo{} for FTs, we first prove an auxiliary lemma that simplifies Definition \ref{def:AC_app}.1. It relies mainly on Theorem \ref{thm:mono_app}; in light of Proposition \ref{prop:singleton}, we only consider singleton causes. We also only consider causes for $X_{\R} = 1$ in order to streamline our statements. 

\begin{lemma} \label{lem:AC_app}
Let $G$ be the graph causal model of a fault tree, and let $\vec{u}$ be a context. Let $c,v \in \V$ and let $x_c \in \BB$. Then $\X_c = x_c$ satisfies \ACo{} for $\X_v = 1$ iff the following hold:
\begin{enumerate}
\item $\struc{G}(\vec{u},c) = x_c = 1$.
\item There exists a partition $\V = Z \sqcup W$ such that $C \subseteq Z$ and such that:
\begin{enumerate}
\item $\struc{G}(\vec{u},v)[\X_c \leftarrow 0,\X_W \leftarrow y_W] = 0$;
\item Let $Z' = \{z \in Z \mid \Phi_{G}(\vec{u},z) = 0\}$ (so $c \notin Z')$. Then 
\[\struc{G}(\vec{u},v)[\X_c \leftarrow 1,\X_{Z'}\leftarrow \vec{0},\X_W \leftarrow y_W] = 1.\]
\end{enumerate}
\end{enumerate}
\end{lemma}

\begin{proof}
Suppose $c,x_c$ satisfy Definition \ref{def:AC_app}. By monotonicity, $\struc{G}(\vec{u},v)[\X_c \leftarrow 1,\X_W \leftarrow y_W] \geq \struc{G}(\vec{u},v)[\X_W \leftarrow y_W] = 1$, so $y_c = 0$; hence $x_c$ must be 1 as it cannot be equal to $y_c$; this shows 1 of Lemma \ref{lem:AC_app}. Similarly, $\struc{G}(\vec{u},v)[\X_c \leftarrow 0,\X_W \leftarrow y_W]$ cannot be equal to $x_v = 1$, so it must be $0$. This shows 2a of Lemma \ref{lem:AC_app}. Finally, 2b is a special case of 2b of Definition \ref{def:AC_app}.

Now suppose $c,x_c$ satisfy the conditions of Lemma \ref{lem:AC_app}. Then condition 3 of Definition \ref{def:AC_app} is automatically satisfied, and conditions 1 and 2a follow from conditions 1 and 2a of the lemma. For condition 2b, let $Z'' \subseteq Z$ be any subset containing $c$, and for $i \in \BB$, define $Z''_i = \{z \in Z''\mid\Phi_{G}(\vec{u},z) = i\}$. Then $c \in Z''_1$ and $Z''_0 \subseteq Z'$, so
\begin{align*}
1 &= \struc{G}(\vec{u},v)[\X_c \leftarrow 1,\X_{Z'}\leftarrow \vec{0},\X_W \leftarrow y_W] \\
&\leq \struc{G}(\vec{u},v)[\X_{Z''_1} \leftarrow \vec{1},\X_{Z'}\leftarrow \vec{0},\X_W \leftarrow y_W] \\
&\leq \struc{G}(\vec{u},v)[\X_{Z''_1} \leftarrow \vec{1},\X_{Z''_0}\leftarrow \vec{0},\X_W \leftarrow y_W] \\
&= \struc{G}(\vec{u},v)[\X_{Z''} \leftarrow x_{Z''}^*,\X_W \leftarrow y_W],
\end{align*}
which shows 2b.
\end{proof}

We use this lemma to prove the following analogon of Theorem \ref{thm:maino} for graph causal models:

\begin{theorem} \label{thm:path_app} Let $T = (V,E,\gamma)$ be a FT, and let $\vec{u}$ be a context of the graph causal model $G_T$. Let $v \in V$. Then $X_v = 1$ satisfies \ACo{} for $X_{\R} = 1$ if and only if there is a path $v = v_0,\ldots,v_n = \R$ in the directed graph $(V,E)$ such that:
\begin{enumerate}
\item $\struc{G_T}(\vec{u},v_i) = 1$ for all $0 \leq i \leq n$.
\item For all $i,j>0$, if $\ch(v_i) \cap \ch(v_j) \neq \varnothing$, then $\gamma(v_i) = \gamma(v_j)$.
\end{enumerate}
\end{theorem}

\begin{proof}
First, suppose that such a path exists. Clearly the first condition of Lemma \ref{lem:AC_app} is satisfied. For the second condition, we take $W = \bigcup_{i>0}(\ch(v_i)\setminus\{v_{i-1}\})$, and $Z$ the remaining variables. For $w \in W$, we let $y_w = \tzero$ if $w \in \ch(v_i)$ for some $v_i$ with $\gamma(v_i) = \tOR$ and $i > 0$, and we let $y_w = \tone $ if $w \in \ch(v_i)$ for some $v_i$ with $\gamma(v_i) = \tAND$ and $i > 0$. By assumption 2 these two will never be true simultaneously.

Going bottom-up, it is clear that
\[
\struc{G_T}(\vec{u},v_i)[\X_v \leftarrow \tzero, \X_W \leftarrow y_W] = \tzero.
\]
for all $i$; in particular $\struc{G_T}(\vec{u},\R)[\X_v \leftarrow \tzero, \X_W \leftarrow y_W] = \tzero$; hence we satisfy 2a. Condition 2b follows from the fact that by construction, $Z' \cap (\{v_1,\ldots,v_r\}) = \varnothing$ (where $Z'$ is as in Lemma \ref{lem:AC_app}), and from the fact that we again show bottom-up that
\[
\struc{G_T}(\vec{u},v_i)[\X_v \leftarrow 1, \X_W \leftarrow y_W] = 1.
\]

Now suppose that $X_v = \tone$ is an actual cause. Suppose that $v = \R$. By Definition \ref{def:AC_app}.1, $\struc{G_T}(\vec{u},v) = \tone$, so the path $v = p_0 = \R$ satisfies our conditions.

Now suppose that $\R \neq v$. Assume $\gamma(\R) = \tOR$, and fix the $W$ and $y_W$ of condition 2 of Lemma \ref{lem:AC_app}. Each $w \in W \cap \ch(\R)$ has to satisfy $y_w = \tzero$ to ensure $\struc{G_T}(\vec{u},\R)[\X_v \leftarrow y_v,\X_W \leftarrow y_W] = \tzero$. Furthermore, let 
\[K = \{z \in Z \cap \ch(\R) \mid \struc{G_T}(\vec{u},z) = \tone\}.\] 
By condition 2b we have $K \neq \varnothing$. For each $k \in K$, we have $\struc{G_T}(\vec{u},k)[\X_v \leftarrow \tzero, \X_W \leftarrow y_W] = \tzero$, for otherwise condition 2a is not satisfied. Let $Z'$ be as in Lemma \ref{lem:AC_app}. If $\struc{G_T}(\vec{u},k)[\X_v \leftarrow 1, \X_{Z'} \leftarrow \vec{0}, \X_W \leftarrow y_W] = 0$ for all $k \in K$, then also $\struc{G_T}(\vec{u},\R)[\X_v \leftarrow 1, \X_{Z'} \leftarrow \vec{0}, \X_W \leftarrow y_W] = 0$, which contradicts condition 2b of Lemma \ref{lem:AC_app}. We conclude that there exists a $k \in K$ such that
\begin{align}
\struc{G_T}(\vec{u},k)[X_v \leftarrow 1, \X_{Z'} \leftarrow \vec{0}, \X_W \leftarrow y_W] &= \tone. \label{eq:pf1}
\end{align}
The set $K$ is partially ordered, with $k \preceq k'$ if there exists a directed path from $k$ to $k'$. Now choose $k \in K$ such that \eqref{eq:pf1} is satisfied and $k$ is minimal w.r.t. $\preceq$. Let $J = \ch(\R) \setminus \{k\}$, and let $W' = W \cup J$, $y_j = \tzero$ for all $j \in J$. Then by Theorem \ref{thm:mono_app},
\begin{align}
\struc{G_T}(\vec{u},k)[\X_v \leftarrow \tzero, \X_{W'} \leftarrow y_{W'}] \leq \struc{G_T}(\vec{u},k)[\X_v \leftarrow \tzero, \X_{W} \leftarrow y_{W}] = \tzero, \label{eq:pf2}
\end{align}
so $\struc{G_T}(\vec{u},\R)[\X_v \leftarrow \tzero, \X_{W'} \leftarrow y_{W'}] = \tzero$. Furthermore, if $j \in J$ is an (indirect) predecessor of $k$, then $\struc{\mathcal{M}}(\vec{u},j)[\X_L \leftarrow y_L,\X_W \leftarrow y_W] = \tzero$ by the fact that we chose $k$ minimal. Hence enforcing $\X_j \leftarrow \tzero$ does not change the value of $k$, i.e.,
\begin{align}
\struc{G_T}(\vec{u},k)[\X_{Z'} \leftarrow y_{Z'},\X_{W'} \leftarrow y_{W'}] = \struc{G_T}(\vec{u},k)[\X_{Z'} \leftarrow y_{Z'},\X_{W} \leftarrow y_{W}] &= \tone. \label{eq:pf3}
\end{align}
It follows that $\struc{G_T}(\vec{u},\R)[\X_{Z'} \leftarrow y_{Z'},\X_{W'} \leftarrow y_{W'}] = \tone$. Thus we conclude that $W'$ and $y_{W'}$ satisfy condition 2 of Definition \ref{def:AC_app}, so we extend $W$ to $W'$ without loss of generality. Then Equations \eqref{eq:pf2} and \eqref{eq:pf3} show that $\X_v = \tone$ is an actual cause for $\X_k = \tone$, taking $W'$ and $y_{W'}$.

If $\gamma(\R) = \tAND$, a similar argument, with $y_j = \tone$ for all $j \in J$, similarly extends $W$ and finds an input $k$ of $\R$ such that $\X_v = \tone$ is an actual cause of $k = \tone$. Furthermore, every $k$ we find has $v$ as an ancestor, otherwise setting $X_v$ would not change its value. We can now repeat the argument, replacing $\R$ by $k$, and continue doing so until we reach a $v$. Enumerating the path bottom-up we get our $v_0,\ldots,v_n$. By how we choose each $k$, each variable on the path satisfies $\struc{G_T}(\vec{u},v_i) = \tone$. Finally, suppose that $v_i,v_j$ with $\gamma(v_i) \neq \gamma(v_j)$ share an input. Since we set all non-path inputs of the path to $\tzero$ for OR-gates and to $\tone$ for AND-gates, the shared input must be a $v_t$ with $v_t \preceq v_i, v_j$. However, this contradicts the fact that each $k$ is chosen minimally with respect to $\preceq$. We conclude that our path satisfies the conditions of the Theorem.
\end{proof}

\begin{proof}[Theorem \ref{thm:mainm}]
This now follows directly from Theorem \ref{thm:path_app} and Lemma \ref{lem:obv3}.
\end{proof}

\subsection{Proof of Theorem \ref{thm:mainu}}

Theorem \ref{thm:mainu} can be phrased in the language of graph causal models as follows:

\begin{theorem} \label{thm:acu_app} Let $T = (V,E,\gamma)$ be a FT, and let $\vec{u}$ be a context of the graph causal model $G_T$. Let $v \in V$. Then $X_v = 1$ satisfies \ACo{} for $X_{\R} = 1$ if and only if there is a path $v = v_0,\ldots,v_n = \R$ in the directed graph $(V,E)$ such that:
\begin{enumerate}
\item $\struc{G_T}(\vec{u},v_i) = 1$ for all $0 \leq i \leq n$.
\item For all $i,j>0$, if $\ch(v_i) \cap \ch(v_j) \neq \varnothing$, then $\gamma(v_i) = \gamma(v_j)$.
\item Let $D = \bigcup_{i\colon \gamma(v_i) = \tOR}(\ch(v_i)\setminus\{v_{i-1}\})$. Then $\struc{G_T}(\vec{u},v_i)[X_D \leftarrow \vec{0}] = 1$.
\end{enumerate}
\end{theorem}

The proof of Theorem \ref{thm:acu_app} is to a large extent analogous to that of Theorem \ref{thm:path_app}; we need two extra ingredients. The first is that \ACu{} only admits singleton causes:

\begin{theorem} \label{thm:singletonu_app}
Let $G$ be a graph causal model derived from a FT, let $\vec{u}$ be a context, and let $X_C = x_C$ be an actual cause of $X_v = x_v$ under \ACu. Then $|C| = 1$.
\end{theorem}

\begin{proof}
First, if $c \in C$, we may replace $c$ with a BE, with context value $\struc{G}(\vec{u},c)$. This does not affect \ACu, so without loss of generality $C \subseteq \BE{}$. We also assume $x_v = 1$; the case that $x_v = 0$ is completely analogous.

Next, suppose that there is a $c$ such that $x_c = y_c$. In that case, let $C' = C \setminus \{c\}$; then since $c$ is a basic event, both conditions (b)i. and (b)ii. of Definition \ref{def:AC_app} are unaffected by replacing $C$ by $C'$. Hence $C$ is not minimal, and $x_c \neq y_c$ in every actual cause.

Now, suppose that there is a $c$ such that $x_c = 0$, $y_c = 1$. Let $C = C' \setminus \{c\}$. Then
\[
\struc{G}(\vec{u},v)[X_{C'}\leftarrow x_{C'}, X_W \leftarrow x_W] \leq \struc{G}(\vec{u},v)[X_{C'}\leftarrow x_{C'},x_C \leftarrow 1, X_W \leftarrow x_W] = 0,
\]
and for all $C' \subseteq Z' \subseteq Z$ and $W' \subseteq W$,
\[
\struc{G}(\vec{u},v)[X_{Z'}\leftarrow x^*_{Z'}, X_W \leftarrow x_W] \leq \struc{G}(\vec{u},v)[X_{Z'}\leftarrow x^*_{Z'},x_C \leftarrow 0, X_W \leftarrow x_W] = 1.
\]
We conclude that $C'$ satisfies Condition (b)i+ii. as well, so $C$ is not minimal. Hence, we can conclude that $x_C = \vec{1}$ and $y_{C} = \vec{0}$.

Let $W_0 = \{w \in W \mid y_w  = 0\}$ and define $W_1$ likewise. Define the function $g\colon \BB^C \rightarrow \BB$ by
\[
g(z_C) = \Phi_G(\vec{u},v)[X_C \leftarrow z_C,X_{W_0} \leftarrow \vec{0}].
\]
By the above, $g(\vec{1}) = 1$, and
\begin{align*}
g(\vec{0}) &= \struc{G}(\vec{u},v)[X_C \leftarrow y_C,X_{W_0}\leftarrow \vec{0}]\\
&\leq \struc{G}(\vec{u},v)[X_C \leftarrow y_C,X_{W_0}\leftarrow \vec{0},W_1 \leftarrow \vec{1}]\\
&= \struc{G}(\vec{u},v)[X_C \leftarrow y_C,X_W \leftarrow W]\\
&= 0.
\end{align*}

Furthermore, by Theorem \ref{thm:mono_app}, $g$ is nondecreasing. Thus, there exists a $c \in C$ and a $z_C \in \BB^C$ with $z_c = 0$ such that $g(z_C) = 0$ and $g(\hat{z}_C) = 1$, where $\hat{z}_{c} = 1$ and $\hat{z}_{c'} = z_{c'}$ for all $c' \neq c$. We claim that $X_c = 1$ satisfies \textbf{AC2}. We take the partitioning $V = Z' \sqcup W'$, with $Z' = (Z \setminus C) \cup \{c\} \cup W_1$ and $W' = W_0 \cup C \setminus \{c\}$ and
\[
y'_w = \begin{cases}
    0, \textrm{ if $w \in W_0$},\\
    z_w, \textrm{ if $w \in C\setminus \{c\}$.}
\end{cases}
\]
For (b)i., note that
\begin{align*}
\struc{G}(\vec{u},v)[X_c \leftarrow 0, X_{W'} \leftarrow y'_{W'}] &= \struc{G}(\vec{u},v)[X_C \leftarrow z_C,X_{W_0}\leftarrow y_{W_0}]\\
&= g(z_C)\\
&= 0.
\end{align*}
For (b)ii., first observe that
\begin{align*}
\struc{G}(\vec{u},v)[X_{c}\leftarrow x^*_{c},X_{W'} \leftarrow y'_{W'}] &= \struc{G}(\vec{u},v)[X_C \leftarrow \hat{z}_C,X_{W_0} \leftarrow y_{W_0}]\\
&= g(\hat{z}_C)\\
&= 1.
\end{align*}
To get to the general statement of 2b, we should be able to remove the setting of variables corresponding to elements of $W_0$ and $C \setminus \{c\}$, and add the setting of variables corresponding to elements of $Z$, without affecting the outcome. Elements of $W_0$ are currently set to 0, so removing them will not decrease the outcome. Elements of $w \in C \setminus \{c\}$ are basic events with $u_w = 1$, which is the value they will take when not set. By monotonicity, removing them will not decrease the outcome. Finally, consider elements of $w \in Z$. Because the reasoning above can also be applied to $w$ instead of $v$, we have
\[
\struc{G}(\vec{u},w)[X_c \leftarrow x_c^*,X_{W''} \leftarrow y_{W''}] \leq x^*_w.
\]
Thus setting $X_w$ to $x_w^*$ will not decrease the outcome. This proves 2b.
\end{proof}

The second ingredient is the following analogon to Lem \ref{lem:AC_app}. The proof is completely analogous, so we omit it.

\begin{lemma} \label{lem:ac-u_app}
Let $G$ be a graph causal model of a fault tree, and let $\vec{u}$ be a context. Let $c,v \in V$, and let $x_v \in \BB$. Then $\X_c = x_c$ is an updated actual cause of $X_v = 1$ iff the following hold:
\begin{enumerate}
\item $\struc{G}(\vec{u},c) = x_c = 1$.
\item There exists a partition $V = Z \sqcup W$ such that $C \subseteq Z$, and values $y_W \in \BB^W$ such that:
\begin{enumerate}
\item $\struc{G}(\vec{u},v)[\X_c \leftarrow 0,\X_W \leftarrow y_W]  = 0$;
\item Let $Z' = \{z \in Z \mid \struc{G}(\vec{u},z) = 0\}$ and $W' = \{w \in W \mid y_w = 0\}$. Then
\[\struc{G}(\vec{u},v)[X_c \leftarrow 1, \X_{Z'}\leftarrow \vec{0},\X_{W'} \leftarrow \vec{0}] = 1.\]
\end{enumerate}
\end{enumerate}
\end{lemma}

Theorem \ref{thm:acu_app} is now proven completely analogous to Theorem \ref{thm:path_app}. 

\begin{proof}[Theorem \ref{thm:mainu}]
This now follows directly from Theorem \ref{thm:acu_app} and Lemma \ref{lem:obv3}.
\end{proof}

\subsection{Proof of Theorem \ref{thm:mainm}}

In the graph causal model terminology, Theorem \ref{thm:mainm} becomes the following:

\begin{theorem}\label{thm:acm_app}
Let $T$ be a FT, let $G_T$ be its graph causal model, and let $\vec{u}$ be a context. Let $C$ be a set of events, and let $x_C \in \BB^C$. Then $X_C = x_C$ is an actual cause under {\ACm} for $X_{\R} = 1$ if and only if the following are satisfied:
\begin{enumerate}
\item $\struc{T}(\vec{u},\R) = 1$;
\item $x_C = \vec{1}$, and $u_v = 1$ for all $v \in C$;
\item If $D = \{v \in V \mid u_v = 1\}$, then $\struc{T}(D \setminus C,\R) = 0$;
\item $C$ is minimal w.r.t. property 3.
\end{enumerate}
\end{theorem}

\begin{proof}
If $C$ satisfies these conditions, then taking $W = \varnothing$ shows that $C$ satisfies Definition \ref{def:AC_app}.3. Conversely, suppose that $\X_{C} = x_C$ satisfies \ACm{} for $\X{\R} = \tone$. \ref{def:AC_app}.3(a) implies \ref{thm:acm_app}.1, and \ref{def:AC_app}.3(c) implies \ref{thm:acm_app}.4. Furthermore, all $v \in C$ with $x_v = 0$ can be removed by monotonicity, hence \ref{thm:acm_app}.2 holds. Finally, by monotonicity we have
\[
\struc{G}(\vec{u},w)[\X_C \leftarrow \vec{\tzero}] \leq \struc{G}(\vec{u},w)
\]
for all $w \in V$; hence
\[
\struc{G_T}(\vec{u},w)[\X_C \leftarrow \vec{\tzero}] \leq \struc{G_T}(\vec{u},w)[\X_C \leftarrow \vec{\tzero},\X_{W} \leftarrow x_W^*]
\]
for all $w \in V$. In particular, we get $\struc{G_T}(\vec{u},\R)[\X_C \leftarrow \vec{\tzero}] = \tzero$, which we translate back to the FT level as $\struc{T}(D\setminus C,\R) = \tzero$.
\end{proof}

\begin{proof}[Theorem \ref{thm:mainm}]
This follows from Theorem \ref{thm:mainm} and Lemma \ref{lem:obv3}.
\end{proof}

\subsection{Proof of Theorem \ref{thm:mcs}}

In terms of graph causal models, Theorem \ref{thm:mcs} can be rephrased as follows: 

\begin{theorem} \label{thm:mcs_app}
Let $T$ be a FT, and let $v$ be a basic event. 
\begin{enumerate}
    \item If there is an MCS $D$ such that $v \in D$, then $\vec{u}^D$ is a context of $G_T$ under which $X_v = 1$ satisfies {\ACo}, {\ACu}, and {\ACm} for $X_{\R} = 1$. 
    \item If there exists a context under which $X_v = 1$ satisfies \ACm{} for $X_{\R} = 1$, then $v$ is relevant.
\end{enumerate}
\end{theorem} 

\begin{proof}
\begin{enumerate}
\item $\struc{G_T}(\vec{u}^D,v) = \struc{G_T}(\vec{u}^D,\R) = 1$ by definition, so we satisfy Definition \ref{def:AC_app}.1(a). As for 1(b), take $W = \varnothing$. Since $D$ is a minimal cut set we get $\struc{G_T}(\vec{u}^D,\R)[\X_v \leftarrow \tzero] = \tzero$. Furthermore, since $W$ is empty, setting variables in $Z$ to their actual value does nothing, so $\struc{G_T}(\vec{u}^D,\R)[Z' \leftarrow \vec{0}] = \struc{G_T}(\vec{u},\R) = \tone$, so we also satisfy 1(b); this proves \ACo{}, and \ACu{} is proven analogously. It is easy to see that this same $\vec{u}^D$ satisfies Theorem \ref{thm:acm_app}, which proves \ACm{}.
\item Suppose that there exists a $D \subseteq \BE{T}$ such that $X_c = 1$ satisfies \ACm{} for $X_{\R} = 1$ under context $\vec{u}^D$. By Theorem \ref{thm:acm_app}.1, $D$ is a cut set. If we take $D$ to be minimal, then $D$ is a minimal cut set, hence $c$ is relevant as $c \in D$. 
\end{enumerate}
\end{proof}

\begin{proof}[Theorem \ref{thm:mcs}]
This follows from Theorem \ref{thm:mcs_app} and Lemma \ref{lem:obv3}.
\end{proof}

\subsection{Proof of Theorem \ref{thm:mcs2}}

In terms of graph causal models, Theorem \ref{thm:mcs2} can be rephrased as follows:

\begin{theorem} \label{thm:mcs->ac-o2_app}

Let $T$ be a FT, and let $\vec{u}$ be a context of $G_T$. Let $v$ be a BE. Consider the following two statements:
\begin{enumerate}
\item $\{v \in \BE{T} \mid u_v = 1\}$ contains an MCS containing $v$;
\item $X_v = 1$ satisfies {\ACu} for $X_{\R} = 1$;
\item $X_v = 1$ satisfies {\ACo} for $X_{\R} = 1$.
\end{enumerate}
Then 1 implies 2, and 2 implies 3. If $T$ is tree-shaped or if $T$ is of disjunctive normal form, then all three are equivalent.
\end{theorem}

\begin{proof}
Note that $2 \Rightarrow 3$ is immediate by Lemma \ref{lem:obv3} and Theorem \ref{thm:acum-acu-aco}; we first prove $1 \Rightarrow 2$. As before, we only need to show that condition (b) of Definition \ref{def:AC_app}.2 holds. For this, consider an MCS $C$ as in the theorem. Now let $W = \BE{} \setminus \{v\}$, with $y_w = \tone$ if and only if $y \in C$. Since $C$ is a minimal cut set, we have 
\[\struc{G_T}(\vec{u},\R)[\X_v \leftarrow \tzero, \X_W \leftarrow y_W] = \tzero.\] 
Furthermore, since $\struc{G_T}(\vec{u},v) = \tone$, we get
\[\struc{G_T}(\vec{u},\R)[\X_W \leftarrow y_W] = \tone.\] 
For any BE $w$ we have 
\[\struc{G_T}(\vec{u},w)[\X_W \leftarrow y_W] \leq u_w.\] 
By monotonicity, it follows via induction that 
\[\struc{G_T}(\vec{u},k)[\X_W \leftarrow y_W] \leq \struc{G_T}(\vec{u},k)\]
for all $k \in V$. In particular $\struc{G_T}(\vec{u},z)[\X_W \leftarrow y_W] \leq y_z^*$ for all $z \in Z$. By monotonicity, this means that forcibly setting $\X_z \leftarrow x_z^*$ has no impact on the fact that $\struc{G_T}(\vec{u},\R)[\X_W \leftarrow y_W] = \tone$. Furthermore, since $C$ happens in $\vec{u}$, we have $y_w \leq \struc{G_T}(\vec{u},w)$ for all $w \in W$. In particular, not setting $X_w \leftarrow y_w$ for some $w$ does not affect $\struc{G_T}(\vec{u},\R)[\X_W \leftarrow y_W] = \tone$ either. This proves condition (b).

For $3 \Rightarrow 1$, we start with a tree-shaped FT. Suppose that $X_v = \tone$ satisfies \ACo{} for $X_{\R} = \tone$. By Theorem \ref{thm:path_app} the unique path from $X_v$ to $X_{\R}$ contains only value $\tone$. We now create a new context $\vec{u}'$, by defining, for each $w \in V$, a value $g(w) \in \BB$ recursively from the root:
\begin{itemize}
\item $g(\R) = \tone$;
\item If $g(w) = \tzero$ and $\gamma(w) \neq \tBE$, then $g(w') = \tzero$ for all $w' \in \ch(w)$;
\item If $g(w) = \tone$ and $\gamma(w) = \tAND$, then $g(w') = \tone$ for all $w' \in \ch(w)$;
\item If $g(w) = \tone$ and $\gamma(w) = \tOR$, then arbitrarily choose one $w' \in \ch(w)$, with the restriction that if $w$ lies on the path from $v$ to $\R$, then $w'$ has to be the previous node on this path (compared to $w$). Then, set $g(w') = \tone$ and $g(w'') = \tzero$ for all $w'' \in \ch(w) \setminus \{w'\}$;
\item If $\gamma(w) = \tBE$, and $w' \in U$ is the corresponding vertex, set $u'_{w'}$ to $g(w)$.
\end{itemize}
Since $T$ is treelike, this defines each $g(w)$ uniquely. Also, by induction bottom-up it is easily shown that $g(w) = \struc{G_T}(\vec{u}',w)$ for all $w$. In particular $\struc{G_T}(\vec{u}',\R) = \tone$, so $C = \{v \in \BE{} \mid u'_v = 1\}$ is a cut set. Furthermore, it can be seen that $g(v) = \tone$, so $v \in C$. It remains to show that $C$ is minimal. To see this, note that by construction every $w$ with $\struc{G_T}(\vec{u}',w) = \tone$ has its minimal number of inputs sets to $\tone$. Thus, if we remove one element $w$ from $C$ to get $\vec{u}''$, it can be shown by induction that every node $w'$ from the path from $w$ to $\R$ satisfies $\struc{G_T}(\vec{u}'',w') = \tzero$. In particular, $\struc{G_T}(\vec{u}'',\R) = \tzero$. Since $w$ was chosen arbitrary, this shows that $C$ is an MCS.

We now turn towards a $T$ in disjunctive normal form. Suppose that $\X_v = \tone$ is an actual cause for $\X_{\R} = \tone$. By Theorem \ref{thm:path_app} there exists a path from $v$ to $\R$ on which all variables evaluate to $\tone$. The middle node on this path represents an MCS containing $v$, all of whose elements are set to 1 in $\vec{u}$.
\end{proof}

\begin{proof}[Theorem \ref{thm:mcs2}]
This follows directly from Theorem \ref{thm:mcs->ac-o2_app} and Lemma \ref{lem:obv3}.
\end{proof}

\subsection{Proof of Theorem \ref{thm:comp}}

We prove the statements for \ACo{}, \ACu{} and \ACm{} separately.

\begin{figure}[t]
\centering
\includegraphics[width=\textwidth]{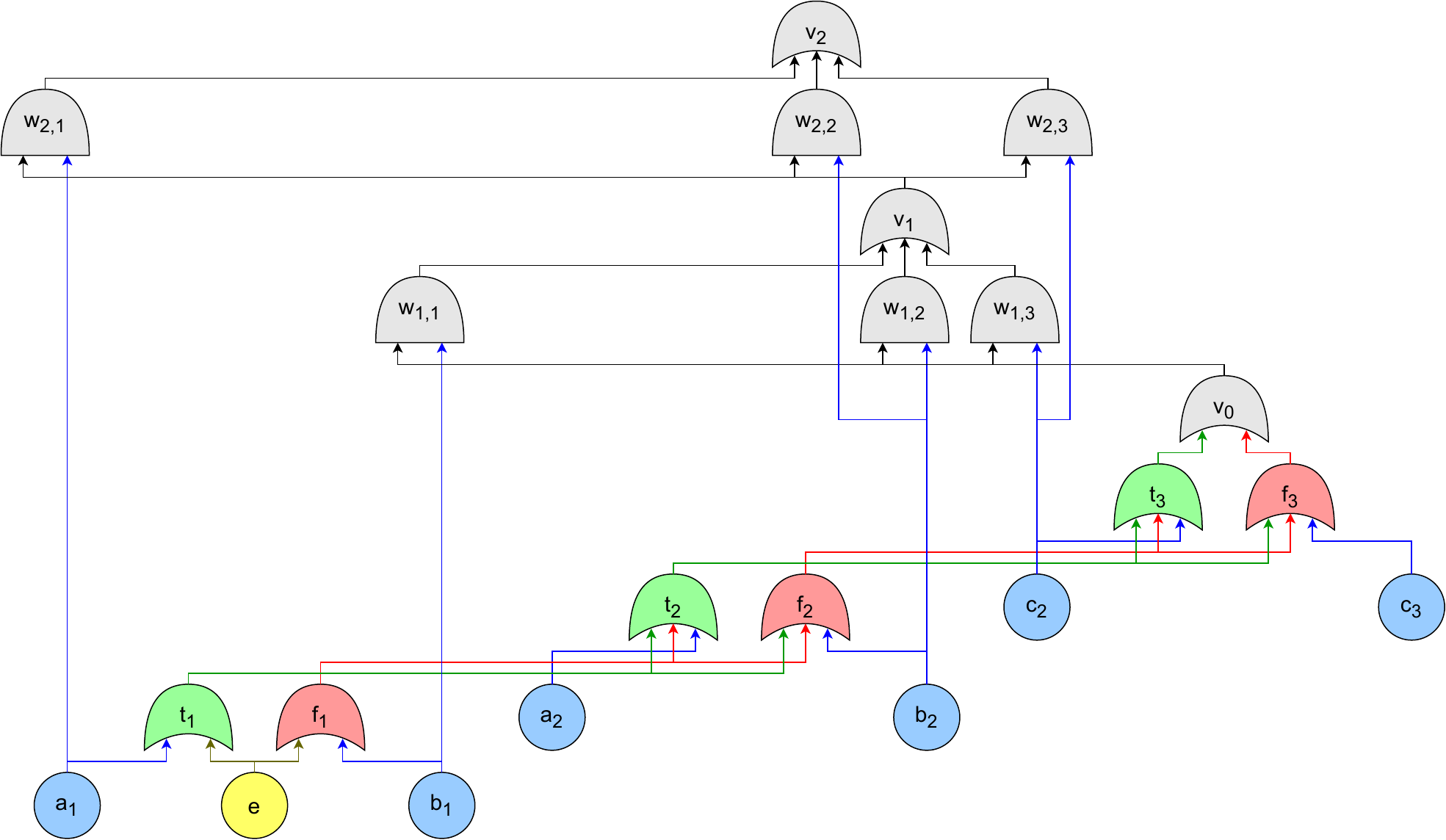}
\caption{An example of the construction of the proof of Theorem \ref{thm:comp1}, for $\varphi = (x_1 \vee x_2 \vee \neg x_3) \wedge (\neg x_1 \vee x_2 \vee \neg x_3)$.} \label{fig:ACoNPC}
\end{figure}

\begin{theorem} \label{thm:comp1}
Problem \ref{prob:dec} is NP-complete for \ACo{}. 
\end{theorem}

\begin{proof}
It is clearly in NP: By Theorem \ref{thm:maino}, a witness for $X_v=1$ being a (singleton) actual cause for $X_w = 1$ is a path from $X \rightarrow Y$ with some properties that take only polynomial time to satisfy. By Theorem \ref{thm:singleton}, there are no non-singleton causes.

To prove NP-hardness, we show that 3-SAT reduces to it. Consider a 3-SAT formula $\varphi = g_1 \wedge \cdots \wedge g_n$, where each $g_i = \ell_{i,1} \vee \ell_{i,2} \vee \ell_{i,3}$, and each $\ell_{i,j} \in \{x_1,\ldots,x_m,\neg x_1,\ldots,\neg x_m\}$. We construct a FT $T_{\varphi}$ consisting of
\begin{itemize}
\item a BE $e$;
\item For each variable $x_k$, OR-gates $t_k$ and $f_k$, BEs $a_k$ and $b_k$, and edges $a_k \rightarrow t_k$ and $b_k \rightarrow f_k$;
\item edges $e \rightarrow t_1$, $e \rightarrow f_1$, and for each $k < m$, edges $t_k \rightarrow t_{k+1}$, $t_k \rightarrow f_{k+1}$, $f_k \rightarrow t_{k+1}$ and $f_k \rightarrow f_{k+1}$;
\item For each clause $g_i = \ell_{i,1} \vee \ell_{i,2} \vee \ell_{i,3}$:
\begin{itemize}
\item An OR-gate $v_i$;
\item An AND-gate $w_{i,j}$ for each $\ell_{i,j}$. If $\ell_{i,j} = x_k$ for some $k$, then there is an edge $b_k \rightarrow w_{i,j}$; if $\ell_{i,j} = \neg x_k$, then there is an edge $a_k \rightarrow w_{i,j}$;
\end{itemize}
\item an OR-gate $v_0$, and edges $t_m \rightarrow v_0$ and $f_m \rightarrow v_0$;
\item edges $v_i \rightarrow w_{i+1,j}$ for each $0 \leq i < n$ and $j \leq 3$.
\end{itemize}
The root of this FT is $v_n$. The resulting FT for $\varphi = (x_1 \vee x_2 \vee \neg x_3) \wedge (\neg x_1 \vee x_2 \vee \neg x_3)$ is depicted in Fig.~\ref{fig:ACoNPC}.

Now let the context $\vec{u}$ be the constant vector $1$. We claim that $X_e = 1$ is a cause for $X_{v_n} = 1$ if and only if $\varphi$ is satisfiable. First, suppose that $X_e = 1$ is a cause for $X_{v_n} = 1$. The path $\pi$ from $e$ to $v_n$ that shows causality must include either $t_k$ or $f_k$ for each $k \leq m$, each $v_i$, and one of the $w_{i,j}$ for each $i$. Given $\pi$, we get a truth assignment given by $x_k \mapsto  1$ iff $t_k$ lies on $\pi$. To show that this truth assignment satisfies $\varphi$, take a clause $g_i$, let $w_{i,j}$ be the input of $v_i$ on $\pi$. If $\ell_{i,j} = x_k$, then $w_{i,j}$ has input $b_k$. This means that $f_k$ cannot lie on $\pi$. Hence $t_k$ does and $x_k \mapsto 1$, and we conclude that $g_i \mapsto 1$. The argument for $\ell_{i,j} = \neg x_k$ is analogous. Since this holds for each $i$, we conclude that $\varphi \mapsto 1$ under this truth assignment.

If $\varphi$ is satisfiable, then we can turn a truth assignment into a path $\pi$ by having $t_k$ lie on $\pi$ if $x_k \mapsto 1$, and having $f_k$ lie in $\pi$ if $x_k \mapsto 0$. Furthermore, for each $i$, choose an $\ell_{i,j}$ that is satisfied, and have $w_{i,j}$ lie on $\pi$. By the same argument as above, this path avoids forbidden pairs (an AND-gate and an OR-gate that share a child), and shows that $X_e = 1$ is a cause for $X_{v_n} = 1$.
\end{proof}

\begin{theorem} \label{thm:comp2}
Problem \ref{prob:dec} is NP-complete for \ACu{}. 
\end{theorem}

\begin{proof}
It is clear that this problem is in NP: by Theorem \ref{thm:mainu}, a witness is again a path (from a singleton cause by Theorem \ref{thm:mainu}), and the required properties can again be checked in polynomial time.

For NP-hardness, we do the same construction as in the proof of Theorem \ref{thm:comp1}. It suffices to show that for every path $\pi\colon e \rightarrow v_n$, setting all children of all OR-gates on $\pi$ to $0$ (except those children that are on $\pi$ themselves) does not change the value of any gate on $\pi$. First, suppose that $t_1$ is on $\pi$. Setting $X_{a_1} \rightarrow 0$ now means that all $w_{i,j}$ for which $\ell_{i,j} = \neg x_1$ are now also set to $0$. However, as argued above, these gates may not be on $\pi$ anyway, so this does not affect the values on $\pi$. Furthermore, the only output of such a gate is $v_i$. Since this is an OR-gate, and it still has a child on $\pi$ (which still has value 1), this is not affected. Hence we can safely set $X_{a_1} \rightarrow 0$ without affecting the gates on $\pi$. Since this is the only non-path input of $t_1$, we conclude that we can set all inputs of $t_1$ to $0$ without affecting the gates on $\pi$. The same reasoning holds for $f_1$.

For a $t_i$ with $i > 1$ on $\pi$, we have the children $a_i$, $t_{i-1}$ and $f_{i-1}$. For $a_i$ the same reasoning as for $a_1$ applies. Either $t_{i-1}$ or $t_{i+1}$ are on $\pi$; if we set the other to $0$, then this has no effects downstream (towards the root), since $f_i$ still has an input with value $1$. Hence again, setting the non-path children of $t_i$ to $0$ has no effect on the path. The same holds for $f_i$.

Any $v_i$ (which is always on $\pi$) has the property that it is the only output of all of its inputs. Therefore setting any of these inputs to $0$ has no effects downstream, since $v_i$ still has an input that still has value $1$.

Overall, we conclude that setting all children of all OR-gates on $\pi$ to $0$, except those that are on $\pi$ itself, has no impact on any of the gates on $\pi$, and in particular not on $v_n$. Hence any path that proves \ACo{} also proves \ACu{} and vice versa. Combined with the proof of Theorem \ref{thm:comp1}, this now proves Theorem \ref{thm:comp2}.
\end{proof}

\begin{theorem} \label{thm:comp3}
Problem \ref{prob:dec} can be solved in polynomial time for \ACm{}.
\end{theorem}

\begin{proof}
Let $T$ be the given FT. Let $C$ be the set of FT nodes represented by $\vec{X}$, and let $T'$ be the FT obtained from $T$ by:
\begin{itemize}
    \item Removing all vertices that are not in the sub-FT with root $w$;
    \item turning each node $v$ in $C$ into a BE, severing the connection to all its inputs, and setting $u_v = 1$;
    \item removing the nodes that no longer have a path to the root.
\end{itemize}
Then $\vec{X} = \vec{1}$ is an actual cause of $X_w = 1$ in $T$ if and only if this is true in $T'$. In $T'$, we are in the situation of Theorem \ref{thm:mainm}, so we just need to check that properties 1-4 of this Theorem are satisfied. Property 2 is automatically satisfied, and properties 1 and 3 just involve computing the structure function bottom-up (for two different inputs), which can be done in linear time. To check property 4, we need to check 1-3 for each set of the form $C\setminus\{x\}$, of which there are only linearly many. We conclude that we can determine whether $\vec{X} = \vec{1}$ is an actual cause of $X_w = 1$ under \ACm{} in quadratic time.
\end{proof}

\subsection{Proof of Theorem \ref{thm:algo}} \label{app:algo}

We first prove correctness. In Remark \ref{rmk:ptu}, it is argued that a path of 1s that avoids all forbidden pairs of $P := P_{T,\vec{u}}$ also avoids all forbidden pairs of $T$. Therefore, $X_v = 1$ is an actual cause (under \ACo{}) for $X_{\R} = 1$ in $T$ if and only if it is so in $P_{T,\vec{u}}$. Now let
\[
\mathcal{H} = \{h \mid (h,l) \in \mathcal{F}(P)\}
\]
be the set of all high elements of forbidden pairs in $P$. A straightforward induction from $\R$ downwards shows that for all $H \subseteq \mathcal{H}$, we have
\begin{align*}
H \in \mathcal{S}(v) \Leftrightarrow& \exists \textrm{ path } \pi\colon v \rightarrow \R \textrm{ s.t. } \pi \cap \mathcal{H} = \varnothing\\
& \textrm{and $\pi$ contains no low elements corresponding to $H$.}
\end{align*}
This means that a path $v \rightarrow \R$ not containing both elements of a forbidden pair exists iff $\mathcal{S}(v) \neq \varnothing$. It is precisely the inputs of such $v$ that have a path as in Theorem \ref{thm:maino}, which shows that line 24 of Algorithm \ref{alg:aco} returns the set of all $v$ for which $X_v = 1$ satisfies \ACo{} for $X_{\R} = 1$.

Let us now consider its complexity. Computing $P_{T,\vec{u}}$ requires computing $\struc{T}(\vec{u},v)$ for all $v$, which takes $\mathcal{O}(|E|)$ time. To find $\mathcal{F}(P)$, we check every pair of edges to see whether they connect a forbidden pair to its shared child. There are $|E|^2$ pairs to check, and each check takes $\mathcal{O}(|V|)$ time since we need to see whether the high and low element are hierarchically related. Finally, to compute $\mathcal{S}(v)$ for each $v$, we pass along an element of $\mathcal{P}(\mathcal{P}(\mathcal{F}(P)))$ along each edge, and this takes $\mathcal{O}(2^{k}|E|)$ time. We conclude that this procedure as a whole has time complexity $\mathcal{O}(|E|^2|V|+2^k|E|)$.

\subsection{Proof of Theorem \ref{thm:algu}}

We first compute $P_{T,\vec{u}}$ ($\mathcal{O}(|E|)$ time per Section \ref{app:algo}), and we compute all its forbidden pairs ($\mathcal{O}(|E|^2)$ time). To determine whether a given vertex $v$ is an actual cause of $X_{\R} = 1$ under \ACu{}, we first check whether $v,\R \in P_{T,\vec{u}}$. Then, we enumerate all $\mathcal{O}(2^{|V|})$ paths $v \rightarrow \R$ in $P_{T,\vec{u}}$. For each such path, we check whether setting all children of OR-gates to $0$ changes the value of $X_{\R}$ by recomputing $P_{T,\vec{u}}$ ($\mathcal{O}(|E|)$ time). Finally, we check if the path contains some forbidden pair ($\mathcal{O}(|V|)$ time per pair for $\mathcal{O}(|E|^2)$ pairs, so $\mathcal{O}(|V||E|^2)$ time). Since we have to do this per path and per $v$, the total computation time is

\[
\mathcal{O}(|E|+|E|^2+2^{|V|}|V|(|E|+|V||E|^2|) = \mathcal{O}(2^{|V|}|V|^2|E|^2).
\]

\subsection{Proof of Theorem \ref{thm:algm}}

By Lemma \ref{lem:mps}, it suffices to construct $\check{P}_{T,\vec{u}}$ ($\mathcal{O}(|E|)$ time per Section \ref{app:algo}): if this does not contain $\R$, then $X_{\R} = 1$ has no causes. If it does, we need to find all minimal cut sets of $\check{P}_{T,\vec{u}}$. The MICSUP algorithm \cite{Vesely1981FaultTreeHandbook} computes all MCSs bottom-up, by storing a set of sets of BEs at each node (for details see \cite{lopuhaa2022efficient}). Thus at every node we have $\mathcal{O}(2^{|\BE{T}|})$ sets that are communicated upwards once over each edge, at which point they are combined; this takes $\mathcal{O}(2^{|\BE{T}|}|E|)$ time. Thus the total time complexity is $\mathcal{O}(|E|+2^{|\BE{T}|}|E|) = \mathcal{O}(2^{|\BE{T}|}|E|)$.

\end{document}